\documentclass[10pt]{article} %
\usepackage[accepted]{tmlr}

\usepackage{amsmath,amsfonts,bm}

\def\eqref#1{equation~\ref{#1}}

\def\1{\bm{1}}

\DeclareMathAlphabet{\mathsfit}{\encodingdefault}{\sfdefault}{m}{sl}
\SetMathAlphabet{\mathsfit}{bold}{\encodingdefault}{\sfdefault}{bx}{n}

\usepackage[hidelinks]{hyperref}
\usepackage{url}

\usepackage{tikz}
\usepackage{amsfonts}
\usepackage{amsthm}
\usepackage[show]{chato-notes}
\usepackage{cleveref}
\usepackage{algorithm}
\usepackage{multirow}
\usepackage{subcaption}
\usepackage{booktabs}
\usepackage{colortbl} %

\definecolor{ForestGreen}{cmyk}{0.864, 0.0, 0.429, 0.396}

\newcommand{\ie}{i.e.,\xspace}
\newcommand{\miattackabbr}{IHA\xspace}
\newcommand{\miattackfull}{Inverse Hessian Attack\xspace}
\newcommand{\attackabbrbolded}{\textbf{I}nverse \textbf{H}essian \textbf{A}ttack\xspace}

\newcommand\shortsection[1]{\vspace{6pt}{\noindent\textbf{#1.}}}
\newcommand\shortersection[1]{\vspace{6pt}{\noindent\em #1.}}

\ifx\proof\undefined
\newenvironment{proof}{\par\noindent{\bf Proof\ }}{\hfill\BlackBox\\[2mm]}
\fi

\theoremstyle{definition}
\newtheorem{theorem}{Theorem}[section]
\newtheorem{assumption}{Assumption}

\newtheorem{lemma}[theorem]{Lemma}
\newtheorem{definition}{Definition}[section]

\title{Do Parameters Reveal More than Loss for \\Membership Inference?}

\author{\name Anshuman Suri \email as9rw@virginia.edu \\
      University of Virginia
      \AND
      \name Xiao Zhang \email xiao.zhang@cispa.de \\
      \addr CISPA Helmholtz Center for Information Security
      \AND
      \name David Evans \email evans@virginia.edu\\
      \addr University of Virginia}

\newcommand{\rebuttal}[1]{#1}
\newcommand{\revision}[1]{#1}

\def\month{12}  %
\def\year{2024} %
\def\openreview{\url{https://openreview.net/forum?id=fmKJfbGKFC}} %

\begin{document}

\maketitle

\begin{abstract}
Membership inference attacks are used as a key tool for disclosure auditing. They aim to infer whether an individual record was used to train a model. While such evaluations are useful to demonstrate risk, they are computationally expensive and often make strong assumptions about potential adversaries' access to models and training environments, and thus do not provide tight bounds on leakage from potential attacks.
We show how prior claims around black-box access being sufficient for optimal membership inference do not hold for stochastic gradient descent, and that optimal membership inference indeed requires white-box access.
Our theoretical results lead to a new white-box inference attack, \textit{\miattackabbr} (\attackabbrbolded), that explicitly uses model parameters by taking advantage of computing inverse-Hessian vector products. Our results show that both auditors and adversaries may be able to benefit from access to model parameters, and we advocate for further research into white-box methods for membership inference. %
\end{abstract}

\section{Introduction}
\label{sec:intro}

Models produced by using machine learning on private training data can leak sensitive information about data used to train or tune the model \citep{salem_sok_2023}.
Researchers study these privacy risks by either designing and evaluating attacks to simulate what motivated adversaries may be able to infer in particular settings or by developing privacy methods that provide strong guarantees, often based on some notion of differential privacy \citep{dwork_calibrating_2006}, that bounds information disclosure from any attack. Although both developing attacks and formal privacy proofs are important, conducting meaningful privacy audits is different from both approaches.
Empirical methods, usually in the form of attack simulations, are inherently limited by the attacks considered and the uncertainty about the possibility of better attacks, while theoretical proofs require many assumptions or result in loose bounds. Further, any claims based on theoretical results depend on careful analysis that the system as implemented is consistent with the theory.
If there is a theoretical result that prescribes an optimal attack, then empirical results with that attack (or approximations of the attack) can offer a more meaningful estimate of privacy risk than is possible with theory or experiments alone. While the theory needs to cover all data distributions, experiments with an optimal attack focus on the actual distribution and given model, resulting in tighter and more relevant privacy evaluations.

Privacy audits can also be important in adversarial contexts, where a regulator or external advocate conducts them to test a released model.
Auditors with elevated model access (such as associated training environments or data) may be able to take advantage of more information to produce better estimates of what an adversary could do without that information.
Auditing is orthogonal to proofs that establish differential privacy bounds or other privacy notions. As outlined by \citet{cummings_challenges_2023}, theoretical bounds may be ``too conservative or inaccurate in some settings'', and it may not always be possible to come up with proofs or theoretical bounds %
that ensure models do not  ``violate disclosure requirements in ways that are not captured by differential privacy''. Empirical auditing can provide a more meaningful measure of privacy leakage for these situations.

The most common disclosure auditing approach today is to conduct membership inference attacks \citep{kumar_ml_2020} and related attacks that attempt to extract specific data \citep{cummings_challenges_2023}.
While membership inference assumes the adversary already knows the full candidate record, it may still constitute a direct privacy risk when revealing the inclusion of a known record in the training data itself, which leaks sensitive information.
In most scenarios, however, membership disclosure by itself is not a serious privacy risk, but rather used as a proxy for understanding information leakage that may result in more serious privacy violations.
Membership inference is simple to define, relatively easy to measure, and aligns well with differential privacy. This has resulted in it being widely used as a method for auditing disclosure risks for machine learning \citep{kumar_ml_2020, yeom_overfitting_2020, kazmi_panoramia_2024, azize_how_2024}.

Prior results on membership inference attacks have largely focused on the \emph{black-box} setting, where the attacker only has input--output access to the target model. This focus has been reinforced by folklore and results demonstrating negligible gains from parameter access (known as \emph{white-box} attacks) \citep{nasr_comprehensive_2018, carlini_membership_2022}. A well-known theoretical result by \citet{sablayrolles_white-box_2019} proves that black-box access is sufficient for optimal membership inference under certain conditions.
This result has been the basis of several subsequent works \citep{ye_enhanced_2022, chaudhari_chameleon_2024}.
However, the assumptions made in its derivation do not hold for most models, including ones trained with stochastic gradient descent (SGD). \revision{This theoretical result has detered researchers from exploring more the potential for inference methods that utilize parameter access, even though the theoretical result does not apply to common machine learning settings.}

\shortsection{Contributions}
\revision{In this work, we revisit previous assumptions surrounding the optimality of membership inference attacks.}
Utilizing recent advances in discrete-time SGD-dynamics \citep{liu_noise_2021,ziyin_strength_2021}, we provide a more accurate formulation of the optimal membership inference attack \rebuttal{that demonstrates the limitations of} the results from  \citet{sablayrolles_white-box_2019}. In particular, we show that the claim that black-box access is sufficient does not hold for models trained using SGD (\Cref{sec:optimal MI SGD}). Our theoretical result also prescribes an attack \revision{that exploits white-box access for auditing membership leakage}, which we call the \miattackfull (\miattackabbr) (\Cref{sec:optimal mia attack}). We empirically demonstrate its effectiveness in simple settings, \revision{showing that it outperforms state-of-the-art reference-model-based and prior white-box attacks (\Cref{sec:experiments}).
Our analyses suggest that the improved auditing performance can be directly attributed to access to the model's parameters.
}

\section{Preliminaries}
\label{sec:preliminaries}

The section provides background on membership inference (\Cref{ssec:mibackground}) and SGD dynamics (\Cref{sec:sgddynamics}).

\subsection{Membership Inference} \label{ssec:mibackground}
Following the framework established by \citet{sablayrolles_white-box_2019}, let $\mathcal{D}$ be a data distribution from which $n$ records $\bm{z}_1, \bm{z}_2, \ldots, \bm{z}_n$ are i.i.d.\ sampled with $\bm{z}_i = (\bm{x}_i, y_i)$ being the $i$-th record. Let $\bm{w}\in\mathbb{R}^d$ be the model parameters produced by some machine learning algorithm on a training dataset $D$. Assume $m_1, m_2, \ldots, m_n$ follow a Bernoulli distribution with $\gamma= \mathbb{P} (m_i=1)$, where $m_i$ is the membership indicator of $\bm{z}_i$ (\ie $m_i = 1$ if $\bm{z}_i\in D$, and $m_i = 0$ otherwise). Given $\bm{w}$, a \emph{membership inference attack}  aims to predict the unknown membership $m_i$ for any given record $\bm{z}_i$.
\begin{definition}[Membership Inference]
\label{def:membership inference}
Let $\bm{w}$ be the parameters of the target model and $\bm{z}_1$ be a record. Inferring the membership of $\bm{z}_1$ to the training set of $\bm{w}$ is equivalent to computing:
\begin{align*}
    \mathcal{M}(\bm{w}, \bm{z}_1) = \mathbb{P}(m_1 = 1 \: | \: \bm{w}, \bm{z}_1).
\end{align*}
\end{definition}
Let $\mathbb{P}(\bm{w}\:|\: \bm{z}_1,\ldots,\bm{z}_n,m_1,\ldots,m_n)$ be the posterior distribution of model parameters produced by some randomized machine learning algorithm (\ie stochastic gradient descent). Applying Bayes' theorem, \citet{sablayrolles_white-box_2019} derived the following explicit formula for $\mathcal{M}(\bm{w}, \bm{z}_1)$.

\begin{lemma} [\citet{sablayrolles_white-box_2019}]
\label{lem:optimial MI likelihood}
Let $\mathcal{T} = \{\bm{z}_2,\ldots,\bm{z}_n, m_2, \ldots, m_n\}$. Given model parameters $\bm{w}$ and a record $\bm{z}_1$, the optimal membership inference is given by:
    \begin{align}
        \mathcal{M}(\bm{w}, \bm{z}_1) = \mathbb{E}_{\mathcal{T}} \bigg[\sigma\bigg(\ln\bigg(\frac{p(\bm{w} \:|\: m_1=1, \bm{z}_1, \mathcal{T})}{p(\bm{w} \:|\: m_1=0, \bm{z}_1, \mathcal{T})}\bigg) + \ln\bigg(\frac{\gamma}{1-\gamma}\bigg)\bigg)\bigg],
    \end{align}
    where $\sigma(u) = (1+\exp(-u))^{-1}$ is the Sigmoid function, and $\gamma = \mathbb{P}(m_1 = 1)$.
\end{lemma}
To use Lemma \ref{lem:optimial MI likelihood}, one needs to characterize the posterior, $\mathbb{P}(\bm{w}\:|\: \bm{z}_1,\ldots,\bm{z}_n,m_1,\ldots,m_n)$, to make explicit the effect of the inferred record $\bm{z}_1$ on the optimal membership inference $\mathcal{M}(\bm{w}, \bm{z}_1)$.
Recent advances in discrete-time SGD dynamics \citep{liu_noise_2021,ziyin_strength_2021} literature can help provide a connection between the posterior and model parameters.

\subsection{Discrete-time SGD Dynamics}
\label{sec:sgddynamics}

A line of theoretical work \citep{welling_bayesian_2011,sato_approximation_2014, stephan_stochastic_2017,liu_noise_2021,ziyin_strength_2021} has analyzed the continuous- and discrete-time dynamics of stochastic gradient methods and provided insights for understanding deep learning generalization. 
Let $L_\mathrm{tot}(\bm{w}) = L(\bm{w}) + \frac{\alpha}{2} \|\bm{w}\|_2^2$ be the $\ell_2$-regularized total loss that we aim to optimize, where $\alpha\geq 0$ denotes the hyperparameter that controls the regularization strength. Consider an SGD algorithm with the following update rule (for $t=1,2,3,\ldots$):
\begin{equation}
    \begin{cases}
    \label{eq:SGD update rule}
        \bm{g}_t & = \nabla L_{\mathrm{tot}}(\bm{w}_{t-1}) + \bm{\eta}_{t-1}; \\
        \bm{h}_t & = \mu\bm{h}_{t-1} + \bm{g}_t; \\
        \bm{w}_t & = \bm{w}_{t-1} - \lambda\bm{h}_t.
    \end{cases}
\end{equation}
Here, $\mu\in[0,1)$ is the momentum, $\lambda>0$ is the learning rate, and
$$\bm{\eta}_t = \frac{1}{S} \sum_{i\in \mathcal{B}_t} (\nabla \ell(\bm{w}_{t}, \bm{z}_i) + \alpha \bm{w}_{t}) - \nabla L_{\mathrm{tot}}(\bm{w}_{t}) = \frac{1}{S} \sum_{i\in \mathcal{B}_t} \nabla \ell(\bm{w}_{t}, \bm{z}_i) - \nabla L(\bm{w}_{t})$$
represents the unbiased mini-batch noise, where $\mathcal{B}_t$ is a randomly sampled batch of examples with size $S$ from the training dataset $D$, and 
$L(\bm{w}) = \frac{1}{n} \sum_{\bm{z}\in D} \ell(\bm{w}, \bm{z})$.

Assuming a model is trained using SGD according to the update rule defined by Equation \ref{eq:SGD update rule} on a \emph{quadratic loss} and arrives at a \emph{stationary state}, \citet{liu_noise_2021} established a theoretical connection between the Hessian matrix $\mathbf{H}$, the asymptotic noise covariance $\mathbf{C} = \lim_{t\rightarrow\infty} \mathbb{E}_{\bm{w}_t} [\mathrm{cov}(\bm\eta_t, \bm\eta_t)]$, and the asymptotic model fluctuation $\mathbf{\Sigma} = \lim_{t\rightarrow\infty}\mathrm{cov}(\bm{w}_t, \bm{w}_t)$. 
We next lay out the two imposed assumptions.

\begin{assumption}[Quadratic Loss]
\label{asmp:quadratic_loss}
    The total loss function $L_\mathrm{tot}(\bm{w})$ is either globally quadratic or locally quadratic close to a local minimum $\bm{w}^*$. Specifically, the loss function can be approximated as:
    \begin{align}
    \label{eq:quadratic_loss}
        L_\mathrm{tot}(\bm{w}) = L_\mathrm{tot}(\bm{w}^*) + \frac{1}{2}(\bm{w} - \bm{w}^*)^\top (\mathbf{H}(\bm{w}^*) + \alpha\mathbf{I}_d) (\bm{w} - \bm{w}^*) + o(\|\bm{w}-\bm{w}^*\|_2^2),
    \end{align}
    where $\bm{w}^*$ is a local minimum, $\mathbf{H}(\bm{w}^*)$ denotes the Hessian matrix at $\bm{w}^*$ with respect to the unregularized loss function $L(\bm{w})$, and $\mathbf{I}_d$ stands for the $d\times d$ identity matrix.
\end{assumption}

\begin{assumption}[Stationary-State]
\label{asmp:stationary_state}
    After a sufficient number of iterations, models trained with SGD defined by Equation \ref{eq:SGD update rule} arrive at a stationary state, \ie the asymptotic model fluctuation $\bm\Sigma$ exists and is finite.
\end{assumption}

Under the above assumptions, \citet{liu_noise_2021} proved the following theorem that describes model fluctuations of discrete SGD in a
quadratic potential with connections to the Hessian matrix and the noise covariance:

\begin{theorem} [SGD Stationary distribution with momentum]
\label{thm:stationary sgd}
Let $\bm{w}$ be updated with SGD defined by the update rule in Equation \ref{eq:SGD update rule} with momentum $\mu\in[0,1)$. Under Assumptions \ref{asmp:quadratic_loss} and \ref{asmp:stationary_state}, if we additionally suppose $\mathbf{C}$ commutes with $\mathbf{H}(\bm{w}^*)$, then the asymptotic model fluctuation satisfies
\begin{align*}
    \bm\Sigma = \bigg[ \frac{\lambda \big(\mathbf{H}(\bm{w}^*)+\alpha\mathbf{I}_d\big)}{1+\mu} \cdot \bigg(2 \mathbf{I}_d - \frac{\lambda \big(\mathbf{H}(\bm{w}^*)+\alpha\mathbf{I}_d\big)}{1+\mu}\bigg)  \bigg]^{-1} \frac{\lambda^2 \mathbf{C}}{1-\mu^2}.
\end{align*}
\end{theorem}
Theorem \ref{thm:stationary sgd} requires the existence of a finite stationary noise covariance and that the loss function is quadratic close to a local minimum, which are both mild assumptions (see \cite{liu_noise_2021} for detailed discussions).

In a follow-up work, \citet{ziyin_strength_2021} further derived the explicit dependence of the finite stationary noise covariance $\mathbf{C}$ on the loss and Hessian around a local minimum $\bm{w}^*$ under certain assumptions:
\begin{theorem}[SGD Noise Covariance]
\label{thm:sgd noise covariance}
    Let $L_\mathrm{tot}(\bm{w}) = L(\bm{w}) + \frac{\alpha}{2} \|\bm{w}\|_2^2$ be the total loss with $\alpha\geq 0$. Assume the model $\bm{w}$ is optimized with SGD defined by Equation \ref{eq:SGD update rule} around a local minimum $\bm{w}^*$. If $L(\bm{w}^*) \neq 0$, then
    \begin{align*}
        \mathbf{C} = \frac{2 L(\bm{w}^*)}{S} \mathbf{H} (\bm{w}^*)  - \frac{\alpha^2}{S} \bm{w}^* \bm{w}^{*\top} + O(S^{-2}) + O(\|\bm{w} - \bm{w}^*\|_2^2) + o(L(\bm{w}^*)),
    \end{align*}
provided that $\bm\Sigma$ is proportional to $S^{-1}$ and $|L(\bm{w}) - \ell(\bm{w}, \bm{z}_i)|$ is small (\ie of order $o(L(\bm{w})$).
\end{theorem}
The first imposed assumption of $\mathbf{\Sigma} = O(S^{-1})$ has been justified by prior works \citep{liu_noise_2021,xie_diffusion_2021,mori2022power}; the second assumption assumes that the current total training loss $L(\bm{w})$ approximates well the individual loss for each record $\ell(\bm{w}, \bm{z}_i)$. Also, note that Theorem \ref{thm:sgd noise covariance} directly implies that the SGD noise covariance $\mathbf{C}$ commutes with the Hessian matrix $\mathbf{H}(\bm{w}^*)$.

Based on the above two theorems and only considering the leading term in the noise covariance, we can immediately derive the following formula for the stationary model fluctuation of SGD: 
\begin{align}
\label{eq:explicit model fluctuation sgd}
    \bm{\Sigma} = \frac{\lambda}{S(1-\mu)} \bigg(2L(\bm{w}^*)\mathbf{H}(\bm{w}^*) - \alpha^2 \bm{w}^* \bm{w}^{*\top}\bigg)\bigg(\mathbf{H}(\bm{w}^*) + \alpha\mathbf{I}_d\bigg)^{-1}\bigg(2\mathbf{I}_d - \frac{\lambda}{1+\mu}\big(\mathbf{H}(\bm{w}^*)+\alpha\mathbf{I}_d\big)\bigg)^{-1}.
\end{align}
We remark that if $L(\bm{w}^*) = 0$ (\ie $\bm{w}^*$ is a global minimum), then $\bm\Sigma = \mathbf{0}$. In addition, if the Hessian matrix $(\mathbf{H}(\bm{w}^*) + \alpha \mathbf{I}_d)$ has degenerate rank $r<d$, then $(\mathbf{H}(\bm{w}^*) + \alpha \mathbf{I}_d)^{-1}$ can be replaced by the corresponding Moore-Penrose pseudo inverse. Accordingly, similar results to Equation \ref{eq:explicit model fluctuation sgd} can be obtained by considering the projection space spanned by eigenvectors with non-zero eigenvalues. Section 5 of \citet{ziyin_strength_2021} provides more detailed discussions of the imposed assumptions and the implications of the results.

\section{Black-Box Access is not Sufficient}
\label{sec:optimal MI SGD}

In this section, we examine previous assertions concerning optimal membership inference (\Cref{sec:black_box_claims}) and show, for models trained with SGD, that optimal membership inference requires parameter access (\Cref{sec:optimal mia}). Our theory directly implies an attack (\Cref{sec:optimal mia attack}).

\subsection{Limitations of Claims of Black-Box Optimality}
\label{sec:black_box_claims}

\citet{sablayrolles_white-box_2019} proved the optimality of black-box membership inference under a Bayesian framework. They assume (Equation 1 in \citet{sablayrolles_white-box_2019}) that the posterior distribution of model parameters $\bm{w}$ trained on $\bm{z}_1,\ldots,\bm{z}_n$ with membership $m_1,\ldots,m_n$ follows
\begin{align}
\label{eq:boltzman distribution}
    \mathbb{P}(\bm{w}\mid \bm{z}_1,\ldots, \bm{z}_n) \propto \exp\bigg(-\frac{1}{T}\sum_{i=1}^n m_i\cdot\ell(\bm{w}, \bm{z}_i)\bigg), 
\end{align}
where $T$ is a temperature parameter that captures the stochasticity of the learning algorithm. This assumption makes subsequent derivations of optimal membership inference much easier, but oversimplifies the training dynamics of typical machine learning algorithms such as SGD. Equation \ref{eq:boltzman distribution} assumes that the posterior of $\bm{w}$ follows a Boltzmann distribution that only depends on the training loss. This is desirable for Bayesian posterior inference, where the goal is to provide a sampling strategy for an unknown data distribution given a set of observed data samples. This can be achieved using stochastic gradient Langevin dynamics (SGLD) \citep{welling_bayesian_2011} with shrinking step size $\lambda_t$ (\ie $\lim_{t\rightarrow\infty}\lambda_t = 0$) and by injecting carefully-designed Gaussian noise $\mathcal{N}(\bm{0}, \lambda_t\cdot \mathbf{I}_D)$. However, this special SGLD design differs from the common practice of SGD algorithms used to train neural networks in two key ways:
\begin{enumerate}
    \item SGLD performs all analyses under continuous-time dynamics whereas actual SGD is performed with discrete steps. 
    While related work such as \citet{stephan_stochastic_2017} cast the continuous-time dynamics of SGD as a multivariate \emph{Ornstein-Uhlenbeck} process (similar to SGLD) whose stationary distribution is proven to be Gaussian (Equations 11 and 12 in \cite{stephan_stochastic_2017}), they make additional assumptions such as the noise covariance matrix being independent of model parameters.
    \item SGLD assumes a vanishing learning rate until convergence, whereas SGD is performed with a non-vanishing step size and for a finite number of iterations in practice. The learning rate of SGD is often large, which can cause model dynamics to drift even further from SGLD \citep{ziyin_law_2023}, especially in the discrete-time setting \citep{liu_noise_2021}.
\end{enumerate}

We thus characterize the analytical form of the posterior distribution with respect to model parameters trained with SGD:

\begin{theorem}[Posterior for SGD]
\label{thm:sgd posterior new}
     Assume the same assumptions as used in Theorems \ref{thm:stationary sgd} and \ref{thm:sgd noise covariance}. Let $\bm{w}^*$ be the local minimum that SGD (Equation \ref{eq:SGD update rule}) is converging towards.
    Then, the (conditional) log-probability of observing parameters $\bm{w}$ is given by (up to constants and negligible terms):
    \begin{align*}
        & - \frac{d}{2} \ln L_* + \sum_{i=1}^d \ln\bigg(\frac{\big(2 - \frac{\lambda}{1+\mu}(\sigma_i(\mathbf{H}_*) + \alpha\big)\big(\sigma_i(\mathbf{H}_*) + \alpha\big)}{\sigma_i(\mathbf{H}_*)} \bigg) -\frac{S(1-\mu)}{2\lambda}\bigg(1 - \frac{\lambda\alpha}{1+\mu}\bigg) \cdot \frac{\|\bm{w} - \bm{w}^*\|_2^2}{L_*} \\
        & \qquad - \frac{S(1-\mu)\alpha}{4\lambda} \cdot \bigg(2 - \frac{\lambda\alpha}{1+\mu}\bigg) \cdot \frac{\nabla L(\bm{w})^\top \mathbf{H}_*^{-3}\nabla L(\bm{w})}{L_*} + \frac{S(1-\mu)}{2(1+\mu)}\cdot\frac{L(\bm{w})}{L_*},
    \end{align*}
    where $L_* = L(\bm{w}^*)$, $\mathbf{H}_* = \mathbf{H}(\bm{w}^*)$ and $\sigma_i(\mathbf{H}_*)$ denotes the $i$-th largest eigenvalue of $\mathbf{H}_*$.
\end{theorem}
A proof for \Cref{thm:sgd posterior new} is given in \Cref{app:sgd posterior new}. Theorem \ref{thm:sgd posterior new} suggests that the posterior distribution of model parameters learned by SGD not only relies on the training loss $L(\bm{w})$ but is also crucially dependent on other terms, such as the Hessian structure $\mathbf{H}^*$, the gradient $\nabla L(\bm{w})$ and the $\ell_2$ distance $\|\bm{w} - \bm{w}^*\|_2$, confirming that Equation \ref{eq:boltzman distribution} is insufficient to model the dynamics of a discrete-time SGD algorithm.

\subsection{Optimal Membership Inference under Discrete-time SGD}
\label{sec:optimal mia}

We have explained why the assumption imposed by \citet{sablayrolles_white-box_2019} about the posterior distribution of $\bm{w}$ following a Boltzmann distribution (Equation \ref{eq:boltzman distribution}) does not hold for stochastic gradient
methods typically employed in practice. Next, we prove a theorem that gives an estimate of the optimal membership inference scoring function  by leveraging recent results on discrete-time SGD dynamics \citep{liu_noise_2021,ziyin_strength_2021}. Our derivation is based on the assumptions that the loss achieved at the local minimum is unaffected by removing a single training record and that the Hessian structure remains unchanged.

\begin{assumption}[Similarity at local minimum]
\label{asmp:hessian}
    For any $\mathcal{T}$ and $\bm{z}_1$, let $L_0(\bm{w}) = \frac{1}{n} \sum_{i=2}^n m_i \ell(\bm{w}, \bm{z}_i)$ and $L_1(\bm{w}) = \frac{1}{n} (\ell(\bm{w},\bm{z}_1) + \sum_{i=2}^n m_i \ell(\bm{w}, \bm{z}_i))$.
    When the training dataset differs only by a single data point $\bm{z}_1$, assume that the Hessian matrix structure for models trained with and without the differing point share a similar structure, and the loss function also achieves a similar value at the local minimum:
    \begin{align}
        \label{eq:hessian_and_loss_same}
    \mathbf{H}_* = \mathbf{H}_0(\bm{w}_0^*) = \mathbf{H}_1(\bm{w}_1^*), \quad
        L_* = L_0(\bm{w}_0^*) = L_1(\bm{w}_1^*),
    \end{align}
    where $\bm{w}^*_0$ is the local minimum that SGD with $L_0$ is converging towards, and $\mathbf{H}_0$ denotes the Hessian matrix with respect to $L_0$ (and likewise for $\bm{w}^*_1$, $L_1$, and $\mathbf{H}_1$).
\end{assumption}

As long as the size of the training dataset is sufficient and the excluded training record $\bm{z}_1$ is not a low-probability outlier from the data distribution $\mathcal{D}$, we expect \autoref{asmp:hessian} generally holds for SGD algorithms.
Under \autoref{asmp:hessian} and a few other assumptions imposed in prior literature on discrete-time SGD dynamics \citep{liu_noise_2021,ziyin_strength_2021}, we obtain a theorem (proof is in \Cref{app:optimal membership inference sgd}) that describes the scoring function for an optimal membership-inference adversary:

\begin{theorem}[Optimal Membership-Inference Score]
\label{thm:optimal mia score}
    Given $\bm{w}$ produced by an SGD algorithm defined by Equation \ref{eq:SGD update rule} and a record $\bm{z}_1$, the optimal membership inference $\mathcal{M}(\bm{w}, \bm{z}_1)$ is given by:
    \begin{align}
    \label{eq:optimal mia score}
        \mathbb{E}_{\mathcal{T}} \bigg[\sigma \bigg(\frac{S(1-\mu)}{2nL_*} \cdot \bigg( \frac{\ell(\bm{w}, \bm{z}_1)}{1+\mu} - \frac{1}{\lambda}(I_1 + I_2 + I_3 + I_4) \bigg) + \ln\bigg(\frac{\gamma}{1-\gamma}\bigg) \bigg) \bigg],
    \end{align}
    where $I_1, I_2, I_3$, and $I_4$ are defined as follows:
    \begin{align*}
        I_1 & := \frac{1}{n}\bigg(1 - \frac{\lambda\alpha}{1+\mu}\bigg) \cdot \|\mathbf{H}_*^{-1}\nabla\ell(\bm{w}, \bm{z}_1)\|^2,\\
        I_2 & := 2\bigg(1 - \frac{\lambda\alpha}{1+\mu}\bigg) \cdot \big(\mathbf{H}_*^{-1}\nabla L_0(\bm{w})\big)^\top \big(\mathbf{H}_*^{-1}\nabla\ell(\bm{w}, \bm{z}_1)\big),
        \\
        I_3 & := \frac{\alpha}{2n} \bigg(2 - \frac{\lambda\alpha}{1+\mu}\bigg)\cdot \big(\mathbf{H}_*^{-1}\nabla\ell(\bm{w}, \bm{z}_1)\big)^\top \big(\mathbf{H}_*^{-1}\big(\mathbf{H}_*^{-1}\nabla\ell(\bm{w}, \bm{z}_1)\big)\big), \\
        I_4 & := \alpha \bigg(2 - \frac{\lambda\alpha}{1+\mu}\bigg) \cdot \big(\mathbf{H}_*^{-1}\nabla L_0(\bm{w})\big)^\top \big(\mathbf{H}_*^{-1}\big(\mathbf{H}_*^{-1}\nabla\ell(\bm{w}, \bm{z}_1)\big)\big).
    \end{align*}
\end{theorem}
Here, $L_*$ and $\mathbf{H}_*$ are defined in Assumption \ref{asmp:hessian}, and are dependent on $\mathcal{T}$.
Here, $\mathcal{T}$ refers to the set of both member and non-member records along with their corresponding membership indicators, as defined in \Cref{lem:optimial MI likelihood}. Note that computing the optimal score requires access to the Hessian and model gradients, both of which require access to the model parameters. In fact, knowledge of the learning rate $\lambda$, momentum $\mu$, and regularization parameter $\alpha$ are also required, thus requiring complete knowledge of the training setup of the target model. Thus, black-box access is \textbf{\emph{not}} sufficient for optimal membership inference. 

The first two additional terms $I_1$ and $I_2$ can be interpreted as the magnitude and direction, respectively, of a Newtonian step for the given record $\bm{z}_1$. The first term $I_1$ characterizes the influence magnitude in $\ell_2$-norm of upweighting $\bm{z_1}$ on the model parameters close to the local minimum~\citep{koh_understanding_2017}, while the second term captures the alignment between the influence of $\bm{z_1}$ and the averaged influence of the remaining training data. A larger \emph{influence magnitude} of $\bm{z}_1$ or an increased \emph{influence alignment} suggests a higher risk of membership inference. We remark that the notion of a self-influence function introduced in \cite{cohen_membership_2024}
naturally relates to $I_1$, suggesting a similar insight to ours that better membership inference attacks can be designed by leveraging the influence function of the inferred record. 
The last two additional terms, $I_3$ and $I_4$, originate from the extra $L_2$ regularization term imposed on the training loss of SGD (Section \ref{sec:sgddynamics}). When the regularization parameter $\alpha$ is a very small positive constant, the effects of $I_3$ and $I_4$ on optimal membership inference will be negligible, particularly compared to those of $I_1$ and $I_2$.

\subsection{\miattackfull}
\label{sec:optimal mia attack}

While \Cref{thm:optimal mia score} directly prescribes an optimal membership inference adversary, computing the expectation over all possible models trained using the rest of the training data is infeasible. Our definition of optimal membership inference corresponds to the true leakage of the model (as defined in Section 3.2 of \citet{ye_enhanced_2022}). It utilizes worst-case adversary knowledge (membership of all other training records) and white-box access to estimate the influence of the target record, similar to how empirical attacks such as LiRA \citep{carlini_membership_2022} and RMIA \citep{zarifzadeh_low-cost_2023} use reference models to account for atypical examples.

Specifically, making use of the insight of \Cref{thm:optimal mia score}, we propose a scoring function based on the terms inside the expectation in Equation \ref{eq:optimal mia score}:
\begin{align}
\label{eq:inverse hessian attack}
    \text{\miattackabbr}(\bm{z}_1) := \frac{\ell(\bm{w}, \bm{z}_1)}{1+\mu} - \frac{1}{\lambda} \big(I_1 + I_2 + I_3 + I_4\big).
\end{align}
The score, $\text{\miattackabbr}(\bm{z}_1)$, for some given record $\bm{z}_1$, can be used as the probability of $\bm{z}_1$ being a member. This serves directly as a useful attack for privacy auditing, without needing to train any reference models. Not having to train reference models offers significant advantages. It helps auditors avoid additional computational costs and, more importantly, eliminates the need for trainers to reserve hold-out data for reference model training. This is particularly beneficial when data availability is a constraint for privacy auditing methods relying on reference models. While the absence of a negative sign
with the loss function (like in LOSS) in $\text{\miattackabbr}(\bm{z}_1)$ may seem counter-intuitive at first glance, it can be rewritten such that it is proportional to the negation of the loss function (Appendix \ref{app:connection_with_loss}).

While membership leakage is typically evaluated on a fixed dataset, the theoretical notion of optimal membership inference is defined for a much larger space. This space encompasses a broad distribution of possible data (including both member and non-member records) and the corresponding models trained on various splits of these datasets. In practice, it's challenging to realize this larger space, but if we could define true positive rates (TPRs) at low false positive rates (FPRs) with respect to this comprehensive space, they would empirically correspond to the optimal membership inference attack. This theoretical framework provides a more robust understanding of membership inference, though its practical implementation remains a significant challenge in privacy auditing.

The performance of our audit is also influenced by other factors, such as how efficiently and accurately the inverse-Hessian vector products (iHVPs) can be computed and to what degree our assumptions hold (particularly \autoref{asmp:hessian}, which requires the Hessian and loss at local minima being unaffected by the exclusion of a single datapoint).

\section{Experiments}
\label{sec:experiments}

To evaluate \miattackabbr, we efficiently pre-compute $\nabla L_1(\bm{w})$ to facilitate the computation of $\nabla L_0(\bm{w})$ for any given target record $\bm{z}_1$. For accurate Hessian matrix computations, we address the issue of ill-conditioning due to near-zero and small negative eigenvalues by either damping or using low-rank approximations (damping-based conditioning seems to perform best; see \Cref{app:our_attack_imp_details} for details). To support larger models where direct Hessian computation is infeasible, we extend our method to use approximation methods based on Conjugate Gradients for iHVP computation \citep{koh_understanding_2017}.
Our implementation for reproducing all the experiments is available as open-source code at \url{https://github.com/iamgroot42/auditingmi}.

\Cref{sec:experimentsetup} describes the baseline attacks, datasets, and models we use for our experiments.  \Cref{sec:results} summarizes our results, showing that \miattackabbr provides a robust privacy auditing baseline, matching or exceeding the performance of current state-of-the-art attacks including attacks that use reference models. This is notable since \miattackabbr\ does ont require training any reference models or the use of hold-out data.

For a given false positive rate (FPR), a threshold is computed using scores for non-members, which is then used to compute the corresponding true positive rate (TPR). This is repeated for multiple FPRs to generate the corresponding ROC curve, which is used to compute the AUC. This experimental design is commonly used for membership-inference evaluations \citep{yeom_privacy_2018, carlini_membership_2022, ye_enhanced_2022}.

\subsection{Setup}
\label{sec:experimentsetup}

To evaluate \miattackabbr, we compare its performance to state-of-the-art baseline attacks with a representative set of datasets and models.

\shortsection{Baseline Attacks}\label{sec:baselineattacks}
We include LOSS as a baseline that does not use reference models, SIF as it uses iHVP similar to our audit, and LiRA as it is the current state-of-the-art for membership inference. While RMIA \citep{zarifzadeh_low-cost_2023} uses fewer reference models, it achieves performance comparable to LiRA and thus for the sake of performance comparison, it suffices to use LiRA with a large number of reference models. \rebuttal{We describe the underlying access assumptions for these attacks in \Cref{tab:attack_assumptions}.}

\begin{table}[h]
    \centering
    \small
    \begin{tabular}{lccc}
        \toprule
         \multicolumn{1}{c}{Attack} & Model Access & Reference Models Used? & Leave-one-out knowledge? \\
         \midrule
         LOSS & Black-box & No & No\\
         LiRA & Black-box & Yes & No\\
         SIF & White-box & No & No\\[1ex]
         L & Black-box & Yes & Yes\\
         LiRA-L & Black-box & Yes & Yes\\[1ex]
         \miattackabbr (Ours) & White-box & No & Yes\\
         \bottomrule
    \end{tabular}
    \caption{Comparison of attacks based on level of model-access, use of reference models, and knowledge of other training members.}
    \label{tab:attack_assumptions}
\end{table}

\shortersection{LOSS {\rm \citep{yeom_privacy_2018}}}  The negative loss is used in this attack as a direct signal for membership inference.

\shortersection{SIF {\rm \citep{cohen_membership_2024}}} Similar to ours, this attack employs the loss curvature of the target model by computing its Hessian, which is then used to compute a self-influence score. The original attack assigns 0--1 scores to target records. It classifies a given record as a member if its self-influence score is within the specified range and if its predicted class is correct. The latter rule can be ruled out as having many false positives/negatives. Instead of these steps, we choose to use the self-influence as membership scores directly. While the authors used approximation methods for iHVP, we use the exact Hessian for fair comparison.

\shortersection{LiRA {\rm  \citep{carlini_membership_2022}}} There are two variants, \emph{LiRA-Offline}, which uses ``offline'' models to estimate a Gaussian distribution and then performs one-sided hypothesis testing using loss scores, and  \emph{LiRA-Online}, with additionally employs ``online'' models, \ie models whose training data includes the target record. The likelihood ratio for online/offline model score distributions is then used as the score for membership inference. We use LiRA-Online, since it is the strongest of the two variants.

\rebuttal{\shortersection{L-Attack {\rm \citep{ye_enhanced_2022}}} The L-attack operates in a leave-one-out setting, training reference models on $D \setminus \{z\}$ for any given record $z$. It uses loss as the target metric and computes attack thresholds for a desired false positive rate (FPR) by leveraging the distribution of losses obtained from reference models.}

\rebuttal{\shortersection{LiRA-L} We propose combining the LiRA attack for the LOO-setting by utilizing reference models trained under leave-one-out availability, followed by the offline variant of the LiRA attack \citep{carlini_membership_2022}.}

\shortsection{Datasets}
Since we are limited by the computational constraints of computing iHVPs, we restrict our experiments to datasets where small models can perform adequately.  

\shortersection{Purchase-100(S)} The task for this dataset \citep{shokri_membership_2017} is to classify a given purchase into one of 100 categories, given 600 features. We train 2-layer MLPs (32 hidden neurons) with cross-entropy loss, with an average test accuracy of $84\%$. Experiments by \citet{zarifzadeh_low-cost_2023} train larger (4-layer MLP) models on 25\,K samples from Purchase-100, which is much smaller than the actual dataset, which is why we term it Purchase-100(S) (Small). We also demonstrate results with a 4-layer MLP that achieves similar task accuracy.

\shortersection{Purchase-100} For this version, we train models with 80\,K samples. We use the same 2-layer MLP architecture as Purchase-100(S) but achieve a higher test accuracy of $90\%$. Using more data increases the scope for model performance. We report results for Purchase-100 in \Cref{tab:main results} as the corresponding models are less prone to overfitting. For completeness, we report results for Purchase-100(S) in \Cref{app:purchase100s_vs_purchase100}.

\shortersection{MNIST-Odd} We consider the MNIST dataset \citep{lecun_gradient-based_1998}, with the modified task of classifying a given digit image as odd or even. This modified task allows us to train models for binary classification using the regression loss, and is thus highly likely to follow the assumptions made in our theory regarding quadratic behavior for the loss function (\Cref{asmp:quadratic_loss}).
We train a logistic regression model with mean-squared error loss, with an average test loss of $.078$.

\shortersection{FashionMNIST} We use the FashionMNIST \citep{xiao_fashion-mnist_2017} dataset, where the task is to classify a given clothing item image into one of ten categories. We train 2-layer MLPs (6 hidden neurons) with cross-entropy loss, with an average test accuracy of $83\%$.

\shortsection{Models}
We train 128 models in the same way as done in \citet{carlini_membership_2022}, where data from each model is sampled at random from the actual dataset with a 50\% probability. For each target model and target record, there are thus 127 reference models available, half of which are expected to include the target record in the training data. All of our models are trained with momentum ($\mu=0.9$) and regularization ($\alpha=5e^{-4}$), with a learning rate $\lambda=0.01$.
For a given false positive rate (FPR), a threshold is computed using scores for non-members, which is then used to compute the corresponding true positive rate (TPR). This is then repeated for multiple FPRs to generate the corresponding ROC curve, which is used to compute the AUC. This experimental design is commonly used for membership-inference evaluations \citep{yeom_privacy_2018, carlini_membership_2022, ye_enhanced_2022}.

\subsection{Results}
\label{sec:results}

As summarized in \Cref{tab:main results}, \miattackabbr provides a strong privacy auditing baseline that is competitive with current state-of-the-art attacks that require reference models.
This is especially useful, considering that \miattackabbr does not require training any reference models and thus, does not require any hold-out data to train such reference models.
\miattackabbr performs much better than the baselines on tabular data (Purchase-100), and is competitive with the baseline for image-based data (MNIST-Odd, Fashion MNIST).
\revision{By extension, our method also outperforms previous membership inference attacks that specifically utilize parameter access via white-box access \citep{nasr_comprehensive_2018,cohen_membership_2024}, since such methods are outperformed by LiRA \citep{carlini_membership_2022}}.
While tabular data is a more realistic setting for membership inference and the improved performance on Purchase-100 is promising, we leave to future work further investigation of these factors to better understand the performance discrepancies.

\begin{table}[t]
\centering
\caption{Performance of various attacks, reported via attack AUC and true positive rate (TPR) at low false positive rate (FPR). ROC curves for low FPR region are visualized in \Cref{fig:roc} (Appendix).}
{
\begin{tabular}{lcrrcrrcrr}
    \toprule
    \multirow{5}{*}{Attack} & \multicolumn{3}{c}{Purchase-100}
    & \multicolumn{3}{c}{MNIST-Odd}
    & \multicolumn{3}{c}{FashionMNIST}
    \\
    \cmidrule(lr){2-4} \cmidrule(lr){5-7} \cmidrule(lr){8-10} 
    & \multirow{3}{*}{AUC} & \multicolumn{2}{c}{\% TPR@FPR} & \multirow{3}{*}{AUC} & \multicolumn{2}{c}{\% TPR@FPR} & \multirow{3}{*}{AUC} & \multicolumn{2}{c}{\% TPR@FPR} \\
    \cmidrule(lr){3-4} \cmidrule(lr){6-7} \cmidrule(lr){9-10}
    & & 1\% & 0.1\% & & 1\% & 0.1\% & & 1\% & 0.1\% \\
    \midrule
    LOSS %
    & .531 $_{\pm.001}$ & 0.97 & 0.00 & .500 $_{\pm .003}$ & 0.97 & 0.09 & .507 $_{\pm .002}$ & 1.00 & 0.10\\
    SIF %
    & .531 $_{\pm .001}$ & 0.97 & 0.10 & .500 $_{\pm .002}$ & 0.97 & 0.10 & .507 $_{\pm .002}$ & 0.98 & 0.10 \\
    LiRA %
    & .644 $_{\pm .004}$ & 4.70 & 0.98 & \textbf{.568 $_{\pm .005}$} & \textbf{2.77} & \textbf{0.63} & .578 $_{\pm .020}$ & 2.98 & 0.63 \\
    \midrule
    \miattackabbr (Exact) & \textbf{.703} $_{\pm .004}$ & \textbf{13.69} & \textbf{7.52} & .542 $_{\pm .004}$ & 2.61 & 0.51 & \textbf{.594} $_{\pm .018}$ & \textbf{4.06} & \textbf{0.89} \\
    \bottomrule
\end{tabular}
}
\label{tab:main results}
\end{table}

\shortsection{Approximating iHVPs}
\rebuttal{In order to carry out \miattackabbr, an auditor needs to be able to calculate iHVPs and gradients for all training data. While computing gradients is more computationally intensive than simply calculating the loss, the difference is minimal. On the other hand, computing an iHVP involves calculating the Hessian matrix and then inverting it, both of which are computationally expensive processes. Even storing such an inverted Hessian can be problematic ($p\times p$ matrix for a model with $p$ parameters).}
We thus experiment with evaluating \miattackabbr using Conjugate Gradients \citep{koh_understanding_2017} to approximate iHVP. While such approximation does not require computing the Hessian directly, the time taken to compute this term for each record is non-trivial. We thus evaluate this approximation-based method on a random sample of 10$\,$000 records\footnote{For a direct comparison, we recompute metrics for the 10$\,$000 samples on which we use approximate-based variants.} and find that approximation methods retain most of the attack's performance (\Cref{tab:iha approximate}).

\begin{table}[tb]
\centering
\caption{Performance of exact and approximation-based variants of \miattackabbr, reported via attack AUC and true positive rate (TPR) at low false positive rate (FPR). Statistics are computed on 10000 samples.}
    \begin{tabular}{lcrrcrr}
    \toprule 
    \multirow{3}{*}{Dataset} & \multicolumn{3}{c}{Exact} & \multicolumn{3}{c}{CG} \\
    \cmidrule(lr){2-4} \cmidrule(lr){5-7}
    & \multirow{2}{*}{AUC} & \multicolumn{2}{c}{\%TPR@FPR} & \multirow{2}{*}{AUC} & \multicolumn{2}{c}{\%TPR@FPR} \\
    \cmidrule(lr){3-4} \cmidrule(lr){6-7}
    & & 1\% & 0.1\% & & 1\% & 0.1\% \\
    \midrule
    Purchase-100 & .701 $_{\pm .009}$ & 13.74 & 7.56 & .701 $_{\pm .009}$ & 13.72 & 7.55 \\
    MNIST-Odd & .541 $_{\pm .009}$ & 2.76 & 0.43 & .541 $_{\pm .009}$ & 2.76 & 0.43 \\
    FashionMNIST & .593 $_{\pm .018}$ & 4.10 & 0.85 & .592 $_{\pm .018}$ & 4.09 & 0.86 \\
    \bottomrule
    \end{tabular}
    \label{tab:iha approximate}
\end{table}

We emphasize that the purpose of our comparisons is not to claim a better membership inference attack for adversarial use;
the threat models are not comparable, since our attack requires knowledge of all other records $D \setminus \{ \bm{z}_1 \}$ for inferring a given target record $\bm{z}_1$ (relaxing this assumption leads to severe performance degradation, see \Cref{app:approximate l0}). 
Instead, \miattackabbr provides a way to empirically audit models for membership leakage without training reference models, which is desirable in avoiding the need to reserve hold-out data for training reference models. More importantly, our results suggest untapped opportunities in exploring parameter access for stronger privacy audits as well as the possibility of new white-box inference attacks from an adversarial lens.

\subsection{Ablating over terms inside IHA}
\label{sec:iha ablation}

As described in Equation \ref{eq:inverse hessian attack}, calculating \miattackabbr requires computing the loss along with the four additional terms $I_1, I_2, I_3$ and $I_4$. However, the terms $I_3$ and $I_4$ are scaled by $\alpha$ (which is usually very small) and involve an iHVP of an iHVP and thus may be much smaller compared to terms like $I_1, I_2$ and the loss. We explore variants of \miattackabbr, which ignore the terms $I_3$ and $I_4$, to see how they impact auditing performance. We also consider variants that use only the terms $I_1$ and $I_2$ to understand the importance of their contributions to the performance of \miattackabbr, along with the inclusion or not of the loss term to understand the relative importance of parameter-based signals.

\begin{table}[tb]
\centering
\caption{Performance of \miattackabbr on Purchase100-S (MLP-2) when only some of the terms corresponding to Equation \ref{eq:inverse hessian attack} are used. $I_1$ and $I_2$ seem to be responsible for most of the privacy auditing performance.}
    \begin{tabular}{lcrr}
    \toprule 
    \multirow{2}{*}{Terms Used} & \multirow{2}{*}{AUC} & \multicolumn{2}{c}{\%TPR@FPR} \\
    \cmidrule(lr){3-4}
    & & 1\% & 0.1\%\\
    \midrule
    $I_1$ & .591 $_{\pm .003}$ & 1.04 & 0.09 \\
    $I_2$ & .704 $_{\pm .004}$ & 2.63 & 0.70 \\
    $I_1$, $I_2$ & .779 $_{\pm .003}$ & 17.61 & 16.65 \\
    $I_1$, $I_2$, $I_3$, $I_4$ & .779 $_{\pm .003}$ & 17.60 & 16.64 \\
    \midrule
    $\ell(\bm{w}, \bm{z}_1)$, $I_1$ & .594 $_{\pm .003}$ & 1.12 & 0.12 \\
    $\ell(\bm{w}, \bm{z}_1)$, $I_2$ & .686 $_{\pm .004}$ & 1.97 & 0.48 \\
    $\ell(\bm{w}, \bm{z}_1)$, $I_1$, $I_2$ & .791 $_{\pm .003}$ & 19.96 & 19.00 \\
    \midrule
    All & .791 $_{\pm .005}$ & 20.09 & 19.09 \\
    \bottomrule
    \end{tabular}
    \label{tab:iha termwise}
\end{table}

The ablation study presented in \Cref{tab:iha termwise} reveals several key insights about the performance of \miattackabbr. Excluding $I_3$ and $I_4$ has negligible impact on auditing performance, even for low-FPR scenarios. The combination of $I_1$ and $I_2$ alone achieves an AUC of $.779$, and performance is identical when the terms $I_3$ and $I_4$ are also included. Notably, the addition of the loss term $\ell(\bm{w}, \bm{z}_1)$ to $I_1$ and $I_2$ results in a marginal improvement, increasing the AUC to $.791$ and slightly boosting the TPR at both 1\% and 0.1\% FPR.
Interestingly, when examined individually, $I_2$ (AUC $.704$) performs significantly better than $I_1$ (AUC $.591$), suggesting that $I_2$ captures more relevant information for the auditing task. Including the loss term $\ell(\bm{w}, \bm{z}_1)$ has little impact on $I_1$ but harms performance when used wth $I_2$.
These findings indicate that the majority of the attack's effectiveness stems from the inverse Hessian vector products used in $I_1$ and $I_2$, with $I_2$ being particularly important, while the terms involving weight regularization and nested iHVPs ($I_3$ and $I_4$) contribute minimally to the overall performance. While not as impactful as $I_2$, the loss term still provides valuable information for the auditing process. Based on these results, a simplified version of \miattackabbr using only $\ell(\bm{w}, \bm{z}_1)$, $I_1$, and $I_2$ could potentially offer a favorable trade-off between computational efficiency and auditing effectiveness.

\subsection{Comparison with Leave-One-Out Setting}
\label{sec:loo setting}

When targeting a record for inference, \miattackabbr assumes knowledge of all the other $n-1$ records in an $n$-sized dataset. It is possible that the improved performance of \miattackabbr is due to this extra information rather than inherent parameter access. To isolate and analyze these potential sources of increase in leakage, we also assess the performance of a leave-one-out (LOO) membership inference test. We evaluate the L-attack \citep{ye_enhanced_2022} on 1000 samples, training 100 reference models per record for score calibration.
Since targeting each record requires training multiple reference models for the L-attack, evaluating it on a larger sample of data is computationally infeasible.

These results indicate that \miattackabbr outperforms the L-attack, achieving an AUC value of $.791$ compared to $.737$ for the L-attack.\footnote{Our results for the L-attack are
lower than those reported by \citet{ye_enhanced_2022}. For instance, we observe a TPR of $.668$ at $0.3$ FPR, while it was reported to be $.968$ by \citep{ye_enhanced_2022}. This discrepancy arises from our setting, which uses more data and fewer model parameters. We verified our implementation through direct correspondence with the authors and by replicating their results in the original setting, which used a smaller dataset and more parameters, resulting in a model prone to overfitting.}  This suggests that even when controlling for the additional knowledge of all other records in the dataset, the primary source of \miattackabbr's superior performance stems from its access to model parameters rather than just the leave-one-out setup. In a way, \miattackabbr uses parameter access to obviate the need for reference models, as it directly aims to measure the influence of the given target record instead of relying on reference models for score calibration.

Interestingly, in comparison, LiRA achieves an AUC of $.767$, which is lower than \miattackabbr but still higher than the L-attack. This suggests that LiRA, even without the extensive reference model training, is more effective than the L-attack, possibly due to its utilization of both ``in'' and ``out'' models as opposed to just ``out'' models with the L-attack.
\rebuttal{To try and devise a stronger black-box attack for the LOO setting, we extend LiRA to the LOO setting by using models trained on LOO data as reference models. LiRA-Offline under the LOO setting achieves an AUC of $.633$, lower than the L-attack ($.737$).}
Overall, these results demonstrate the promise of \miattackabbr as a privacy auditing tool. It  yields results comparable to that of techniques that train hundreds of reference models (for each target record in the worst case, as in L-attack), without using a single reference model.

\subsection{Inter-Attack Agreement}
\label{sec:attack agreement}

Similar to \citet{ye_enhanced_2022}, we compute the agreement rate between ground-truth membership labels and membership predicted by various attacks to understand the ability of our privacy audit to identify vulnerable data, and demonstrate how it differs from existing attacks. 
\Cref{tab:prediction agreement} presents the agreement rate between ground-truth membership labels and membership predicted by various attacks.

\begin{table}[tb]
    \small
    \caption{Agreement rate between ground truth (GT) membership values, and various attacks for 500 training and 500 testing data points. The upper triangle of the table corresponds to the agreement rates of members, whereas the lower triangle corresponds to the agreement rates of non-members. The experimental setup is Purchase-100(S), with effective FPR $\approx 0.05$ (\protect{\subref{tab:prediction agreement:0.05}}) and effective FPR $\approx 0.3$ (\protect{\subref{tab:prediction agreement:0.3}}).}
    \label{tab:prediction agreement}
    \parbox{.49\linewidth}{
    \centering
    \begin{tabular}{|c|c|c|c|c|c|}
        \hline
         & \textbf{LiRA} & \textbf{L} & \rebuttal{\textbf{LiRA-L}} & \textbf{\miattackabbr} & \textbf{GT} \\
         \hline
         \textbf{LiRA} & \cellcolor{gray} & .794 & .826 & .718 & .220 \\
         \hline
         \textbf{L} & .930 & \cellcolor{gray} & .908 & .712 & .234 \\
         \hline
         \rebuttal{\textbf{LiRA-L}} & .916 & .938 & \cellcolor{gray} & .768 & .154\\
         \hline
         \textbf{\miattackabbr} & .912 & .926 & .916 & \cellcolor{gray} & .230 \\
         \hline
         \textbf{GT} & .954 & .952 & .954 & .958 & \cellcolor{gray} \\
         \hline
    \end{tabular}
    \subcaption{Agreement between methods with FPR 0.05}
    \label{tab:prediction agreement:0.05}
    }
    \hfill
    \parbox{.49\linewidth}{
    \centering
    \begin{tabular}{|c|c|c|c|c|c|}
        \hline
         & \textbf{LiRA} & \textbf{L} & \rebuttal{\textbf{LiRA-L}} & \textbf{\miattackabbr} & \textbf{GT} \\
         \hline
         \textbf{LiRA} & \cellcolor{gray} & .732 & .658 & .436 & .652 \\
         \hline
         \textbf{L} & .724 & \cellcolor{gray} & .766 & .492 & .668 \\
         \hline
         \rebuttal{\textbf{LiRA-L}} & .588 & .660 & \cellcolor{gray} & .462 & .518\\
         \hline
         \textbf{\miattackabbr} & .612 & .660 & .584 & \cellcolor{gray} & .640 \\
         \hline
         \textbf{GT} & .702 & .702 & .706 & .706 & \cellcolor{gray}\\
         \hline
    \end{tabular}
    \subcaption{Agreement between methods with FPR 0.3}
    \label{tab:prediction agreement:0.3}
    }
\end{table}

At a low FPR of $\approx 0.05$, agreement in predictions for non-members is very high between attacks, with agreement rates above $0.91$ for all pairs of attacks. On the other hand, agreement rates for member records are expectedly lower. Interestingly, agreement between LiRA and LiRA-L is higher than \miattackabbr and any other attack. This difference is especially pronounced for a higher FPR (\Cref{tab:prediction agreement:0.3}), where agreement rates are as low as $\approx0.4$ compared to $0.766$ for LiRA and LiRA-L. This is very interesting becaue the corresponding TPRs for \miattackabbr are higher than LiRA and comparable to that of the L-attack, thus suggesting that the records identified by \miattackabbr as being vulnerable are very different from those identified by LiRA or even the L-attack. This also means that a combined (classifying a record as a member only when both attacks classify as a member) attack would have some true positives with a very low FPR.

\subsubsection{Runtime Comparison}
\label{sec:runtime}

\rebuttal{To compare the computational costs of our proposed audit with existing auditing techniques, we analyze runtime and memory usage statistics across different methods, aiming to evaluate efficiency and practicality in real-world privacy audits (\Cref{tab:time_and_mem_consumption}).}

\begin{table}[tb]
    \centering
    \small
    \begin{tabular}{l|cc|ccc}
        \toprule
         Attack & \multicolumn{2}{c}{Time (s)} & \multicolumn{3}{c}{Memory (MB)} \\
         & Precompute & Time/Sample & Precompute & Runtime & $\max$(Precompute, Runtime) \\
         \midrule
         \miattackabbr & 43 $\times$ 60 & 0.16 & 4228 & 1324 & 4228\\
         \miattackabbr (Approx) & 0 & 85 & 0 & 1638 & 1638\\
         LiRA & 192 $\times$ 60 & 0.07 & 276 & 1228 & 1228\\
         LOSS & 0 & 0.004 & 0 & 906 & 906\\
         \bottomrule
    \end{tabular}
    \caption{Runtime and memory statistics for various auditing techniques. Statistics are computed over 100 randomly-selected members for MLP-2 architecture models trained on Purchase-100 dataset.}
    \label{tab:time_and_mem_consumption}
\end{table}

\rebuttal{While the peak memory consumption of IHA is higher than that of LiRA, the approximate version of IHA is not too far from LiRA in terms of memory consumption. Computing the total runtime is a function of the number of samples used for the privacy audit, as a “precompute” step is required for LiRA (training reference models) and exact IHA (computing Hessian).} 
\rebuttal{For instance, computing the audit for 1K samples takes less overall time for \miattackabbr ($\approx$ 1 hour) than it does for LiRA ($\approx$ 3 hours). It should be noted that the Hessian is too large to store on our GPU for \miattackabbr and is thus stored on the CPU, which is also why it is slower.}
\rebuttal{Although improvements may reduce the compute costs, the key advantage of such a privacy audit comes from not having to reserve hold-out (or auxiliary) data. %
Our privacy audit, like the most trivial LOSS attack, only requires member data and some non-member data, whereas other attacks in the literature require shadow/reference models trained on comparable-sized datasets; the limiting factor here is reserving data to train reference models, which a real-world model trainer may not want to do since it reduces the amount of data available for training.}

\section{Conclusion}
\label{sec:conclusion} 

Our theoretical result proves that model parameter access is indeed necessary for optimal membership inference, contrary to previous results derived under unrealistic assumptions and the common belief that optimal membership inference can be achieved with only black-box model access. We propose the \miattackfull inspired by this theory that provides stronger privacy auditing than existing black-box techniques.

\rebuttal{\shortsection{Limitations}} \miattackabbr is not yet practically realizable for most settings due to the computational expense of calculating the Hessian, or even approximating iHVPs. \rebuttal{This restriction poses challenges for real-world cases where fixed compute budgets may be more crucial than the availability of auxiliary data. An auditor might use a subset of parameters to reduce computational costs while performing Hessian-based computations. This aligns with model pruning \citep{liu2018rethinking}, but understanding its impact on membership knowledge within parameters is non-trivial \citep{yuan2022membership}.} We also note that \miattackabbr's performance can be sensitive to the choice of the damping factor, which requires further investigation (\Cref{app:picking epsilon}).

Our conclusion aligns well with recent calls in the literature to consider white-box access for membership inference \citep{cretu2024investigating} and rigorous auditing \citep{casper_black-box_2024}.
\rebuttal{While our theory shows that parameter access is required for optimal membership inference, it remains unclear how much better this is compared to the optimal membership inference attack restricted to black-box access. Our empirical studies suggest the gap is non-trivial, but further study is required to understand the theoretical limit of black-box attacks, which is a non-trivial but interesting direction to explore.}
Exploring the accuracy of iHVP approximation methods to extend \miattackabbr to larger models, along with multi-record inference, are both promising directions for future research.

\subsubsection*{Broader Impact Statement}
The increasing integration of AI in sensitive domains like healthcare, finance, and personal data management highlights the critical importance of privacy. Information leakage from AI models can have severe consequences, making effective privacy auditing a necessary safeguard. Our work contributes to this field by theoretically demonstrating that optimal membership inference attacks require white-box access to model parameters, challenging the adequacy of black-box approaches. We also demonstrate with \miattackfull how this theory can be used to design empirical privacy audits that do not rely on reference models.

We advocate for the development of more sophisticated privacy auditing tools that fully leverage the elevated access typically available to auditors, such as model parameters, to assess privacy leakage efficiently without extensive data and compute resources. We hope our theoretical and empirical results will reinvigorate interest in the privacy research community to explore white-box attacks, for both adversarial and auditing purposes.

\bibliography{main}
\bibliographystyle{tmlr}

\appendix

\section{Related Work}
\label{sec:related}

This section reviews methods in membership inference (black-box and white-box), techniques for privacy auditing to predict and mitigate data leakage, and the dynamics of stochastic gradient descent (SGD) along with inverse Hessian vector products (iHVPs).

\subsection{Membership Inference}

\shortsection{Black-box Membership Inference} Early works on membership inference worked under black-box access, utilizing the model's loss \citep{shokri_membership_2017} on a given datapoint as a signal for membership. Since then there have been several works focusing on different forms of difficulty calibration---accounting for the inherent ``difficulty'' of predicting on a record, irrespective of its presence in train data. This calibration has taken several forms; direct score normalization with reference models \citep{sablayrolles_white-box_2019}, likelihood tests based on score distributions \citep{carlini_membership_2022, zarifzadeh_low-cost_2023, ye_enhanced_2022}, and additional models for predicting difficulty \citep{bertran_scalable_2024}.

\shortsection{White-box Membership Inference}
\citet{nasr_comprehensive_2018} explored white-box access to devise a meta-classifier-based attack that additionally extracts intermediate model activations and gradients to increase leakage but concluded that layers closer to the model's output are more informative for membership inference and report performance not significantly better than a black-box loss-based attack. Recent work by \citet{dealcala_is_2024}, however, makes the opposite observation, with layers closer to the model's input providing noticeably better performance.
Apart from these meta-classifier driven approaches, some works attempt to utilize parameter access much more directly, often utilizing Hessian in one form or another.
\citet{cohen_membership_2024} defined the self-influence of a datapoint $\bm{z}_i$ as (${\bm{g}_i}^\top \mathbf{H}^{-1}\bm{g}_i$) as a signal for membership, using LiSSA \citep{agarwal_second-order_2017} to approximate the iHVP. This has similarities to our result since our optimal membership inference score also involves computing iHVPs. \citet{li2023mope} attempted to measure the sharpness for a given model by evaluating fluctuations in model predictions after adding zero-mean noise to the parameters, a step that is supposed to approximate the trace of the Hessian at the given point.

\subsection{Privacy Auditing}
\citet{ye_leave-one-out_2024} proposed using efficient methods to ``predict'' memorization by not having to run computationally expensive membership inference attacks, with reported speedups of up to 140$x$. They showed how their proposed score (LOOD) correlates well with AUC, corresponding to an extremely strong MIA with all-but-one access to records (L-attack \citep{ye_enhanced_2022}). However, it is unclear if this computed LOOD is directly comparable across models, making it hard to calibrate these scores to compare the leakage from a model relative to another (an important aspect of internal privacy auditing). Their derivations also involve a connection with the Hessian.
\citet{biderman_emergent_2024} studied the problem of forecasting memorization in a model for specific training data and proposed using partially trained versions of the model (or smaller models) as a proxy for their computation. While their results support the need for inexpensive auditing methods, their focus is on predicting memorization early in the training process, while ours relates to auditing fully trained models.
More recently, \citet{tan_parameters_2022} studied the theory behind worst-case membership leakage for the case of linear regression on Gaussian data and derived insights. While this is useful to make an intuitive connection with overfitting, it does not provide a realizable attack or insights for the standard case of models trained with SGD.

\subsection{SGD Dynamics and iHVPs}

\shortsection{SGD Dynamics} 
\citet{stephan_stochastic_2017} approximated the SGD dynamics as an Ornstein-Uhlenbeck process, while \citet{yokoi_bayesian_2019} provided a discrete-time weak-order approximation for SGD based on Itô process and finite moment assumption. However, both works rely on strong assumptions about the gradient noises and require a vanishingly small learning rate, largely deviating from the common practice of SGD. To address the limitations of the aforementioned works, \citet{liu_noise_2021} directly analyzed the discrete-time dynamics of SGD and derived the analytic form of the asymptotic model fluctuation with respect to the asymptotic gradient noise covariance and the Hessian matrix. \citet{ziyin_strength_2021} further generalized the results of \citep{liu_noise_2021} by deriving the exact minibatch noise covariance for discrete-time SGD, which is shown to vary across different kinds of local minima. Our work builds on these advanced theoretical results of discrete-time SGD dynamics but aims to enhance the understanding of optimal membership inference, particularly for models trained with SGD.

\shortsection{iHVPs} Currently literature on approximating inverse-Hessian vector products relies on one of two methods: conjugate gradients \citep{koh_understanding_2017} or LiSSA \citep{agarwal_second-order_2017}. Both approximation methods rely on efficient computation of exact Hessian-vector products, and use forward and backward propagation as sub-routines. While these methods have utility in certain areas, such as influence functions \citep{koh_understanding_2017} and optimization \citep{oldewage2024series}, approximation errors can be non-trivial. For instance, $I_1$ in the formulation of our attack requires a low approximation error in the norm of an iHVP, while $I_2$ simultaneously requires a low approximation error in the direction of the iHVP. Recent work on curvature-aware minimization by \citet{oldewage2024series} proposes another method for efficient iHVP approximation as a subroutine, but the authors observed high approximation errors based on both norm and direction.

\section{Proof for Theorem \ref{thm:sgd posterior new}}
\label{app:sgd posterior new}

\begin{proof}
Recall that $\mathbf{H}_* = \mathbf{H}(\bm{w}^*)$ and $L_* = L(\bm{w}^*)$. According to Theorem \ref{thm:stationary sgd} and Theorem \ref{thm:sgd noise covariance}, we obtain
\begin{align}
\label{eq:explicit model fluctuation sgd re}
    \bm{\Sigma} &= \frac{\lambda}{S(1-\mu)} \bigg(2L_*\mathbf{H}_* - \alpha^2\bm{w}^* \bm{w}^{*\top} \bigg)\bigg(\mathbf{H}_* + \alpha \mathbf{I}_d\bigg)^{-1}\bigg(2\mathbf{I}_d - \frac{\lambda}{1+\mu}(\mathbf{H}_* + \alpha \mathbf{I}_d)\bigg)^{-1},
\end{align}
where we only consider the leading terms in Theorems \ref{thm:stationary sgd} and \ref{thm:sgd noise covariance}. Note that Equation \ref{eq:explicit model fluctuation sgd re} holds when the Hessian matrix $\mathbf{H}_*$ has full rank and $L_* \neq 0$. When the Hessian has degenerated rank such that $\mathrm{rank}(\mathbf{H}_* + \alpha \mathbf{I}_d) = r < d$, the following more generalized result can be derived:
\begin{align*}
    \mathbf{P}_r\bm{\Sigma} = \frac{\lambda}{S(1-\mu)} \bigg(2L_*\mathbf{H}_* - \alpha^2\bm{w}^* \bm{w}^{*\top} \bigg) \bigg(\mathbf{H}_* + \alpha\mathbf{I}_d\bigg)^+\bigg(2\mathbf{I}_d - \frac{\lambda}{1+\mu}\big(\mathbf{H}_* + \alpha \mathbf{I}_d\big)\bigg)^{-1}, %
\end{align*}
where $\mathbf{P}_r= \mathrm{diag}(1, ..., 1, 0, ...0)$ denotes the projection matrix onto non-zero eigenvalues, and $+$ is the Moore-Penrose inverse operator. If $L_* = 0$, meaning $\bm{w}^*$ is a global minimum, then the asymptotic model fluctuation $\bm{\Sigma} = \bm{0}$. For ease of presentation, we assume the Hessian matrix has full rank in the following proof. 

Then, we get:
\begin{align}
\label{eq:variance_inverse_raw}
    \mathbf{\Sigma}^{-1} & = \frac{S(1-\mu)}{\lambda} \bigg(2\mathbf{I}_d - \frac{\lambda}{1+\mu}(\mathbf{H}_* + \alpha \mathbf{I}_d)\bigg)\bigg(\mathbf{H}_* + \alpha \mathbf{I}_d\bigg)\bigg(2L_*\mathbf{H}_* - \alpha^2\bm{w}^* \bm{w}^{*\top} \bigg)^{-1}. 
\end{align}

Using the Sherman–Morrison formula, we obtain: 
\begin{align}
\label{eq:sherman-morrison}
    \bigg(2L_*\mathbf{H}_* - \alpha^2\bm{w}^* \bm{w}^{*\top} \bigg)^{-1} &= \frac{1}{2L_*}\mathbf{H}_*^{-1} + \frac{\alpha^2}{2L_*(2L_* - \bm{w}^{*\top}\mathbf{H}_*^{-1}\bm{w}^*)}\big(\mathbf{H}_*^{-1}\bm{w}^* \bm{w}^{*\top}\mathbf{H}_*^{-1}\big) \nonumber \\
    & = \frac{1}{2L_*}\mathbf{H}_*^{-1} \big(\mathbf{I}_d + O(\alpha^2) \big). 
\end{align}
Leaving out the second-order term $O(\alpha^2)$ in Equation \ref{eq:sherman-morrison} (since the regularization parameter $\alpha$ is a typically small constant in $[0,1)$) and
plugging it back in Equation \ref{eq:variance_inverse_raw}, we get:
\begin{align}
    \mathbf{\Sigma}^{-1} & = \frac{S(1-\mu)}{2\lambda L_*} \bigg(2\mathbf{I}_d - \frac{\lambda}{1+\mu}(\mathbf{H}_* + \alpha \mathbf{I}_d)\bigg)\bigg(\mathbf{H}_* + \alpha \mathbf{I}_d\bigg)\mathbf{H}_*^{-1} \nonumber \\
    & = \frac{S(1-\mu)}{2\lambda L_*} \bigg(2\mathbf{I}_d - \frac{\lambda}{1+\mu}(\mathbf{H}_* + \alpha\mathbf{I}_d) + 2\alpha\mathbf{H}_*^{-1} - \frac{\lambda\alpha}{1+\mu}(\mathbf{I}_d + \alpha\mathbf{H}_*^{-1})\bigg) \nonumber \\
    & = \frac{S(1-\mu)}{2\lambda L_*} \bigg(2\bigg(1 - \frac{\lambda\alpha}{1+\mu}\bigg)\mathbf{I}_d - \bigg(\frac{\lambda}{1+\mu}\bigg)\mathbf{H}_* + \alpha\bigg(2 - \frac{\lambda\alpha}{1+\mu}\bigg)\mathbf{H}_*^{-1}\bigg).
\end{align}

According to Laplace approximation, we can approximate the posterior distribution of $\bm{w}$ given $\bm{w}^*$ as $\mathcal{N}(\bm{w}^*, \bm{\Sigma})$. Therefore, making use of Equation \ref{eq:explicit model fluctuation sgd re}, we can derive the explicit formula of the log-posterior distribution as:
\begin{align}
\label{eq:log likelihood explicit form}
    \nonumber & \ln{p(\bm{w} | \bm{w}^*)} =  -\frac{d}{2} \ln(2\pi) - \frac{1}{2} \ln \det(\bm{\Sigma}) - \frac{1}{2} (\bm{w} - \bm{w}^*)^\top \bm{\Sigma}^{-1} (\bm{w} - \bm{w}^*) \\
    \nonumber & \quad= \frac{1}{2}\sum_{i=1}^d \ln\bigg(\frac{\big(2 - \frac{\lambda}{1+\mu}(\sigma_i(\mathbf{H}_*) + \alpha)\big)\big(\sigma_i(\mathbf{H}_*) + \alpha\big)}{2L_*\sigma_i(\mathbf{H}_*)} \bigg) + \text{const.} \\
    \nonumber & \qquad -\frac{S(1-\mu)}{4\lambda L_*}\bigg[ 2\bigg(1 - \frac{\lambda\alpha}{1+\mu}\bigg)\|\bm{w} - \bm{w}^*\|_2^2 + \alpha\bigg(2 - \frac{\lambda\alpha}{1+\mu}\bigg)(\bm{w} - \bm{w}^*)^\top \mathbf{H}_*^{-1}(\bm{w} - \bm{w}^*) \\ 
    &\qquad \quad - \frac{\lambda}{1+\mu}(\bm{w} - \bm{w}^*)^\top \mathbf{H}_*(\bm{w} - \bm{w}^*) \bigg]. 
\end{align}
Note that we assume the model parameterized by $\bm{w}$ is converging towards some local minima $\bm{w}^*$ using stochastic gradient descent, where the loss landscape around the local minima $\bm{w}^*$ has a quadratic structure. 

Therefore, we can perform second-order Taylor expansion of the loss function at $\bm{w}^*$ as follow:
\begin{align}
\label{eq:second order taylor expansion}
L(\bm{w}) = L(\bm{w}^*) + \frac{1}{2}(\bm{w} - \bm{w}^*)^\top \mathbf{H}_* (\bm{w} - \bm{w}^*) + o(\|\bm{w} - \bm{w}^*\|_2^2),
\end{align}
which further suggests that $\nabla L(\bm{w}) = \mathbf{H}_*(\bm{w} - \bm{w}^*)$ if taking gradient with respect to $\bm{w}$ for both sides of Equation \ref{eq:second order taylor expansion}.
Plugging Equation \ref{eq:second order taylor expansion} into Equation \ref{eq:log likelihood explicit form}, we further obtain:
\begin{align}
    \nonumber & \ln{p(\bm{w} | \bm{w}^*)} = \frac{1}{2}\sum_{i=1}^d \ln\bigg(\frac{\big(2 - \frac{\lambda}{1+\mu}(\sigma_i(\mathbf{H}_*) + \alpha)\big)\big(\sigma_i(\mathbf{H}_*) + \alpha\big)}{2L_*\sigma_i(\mathbf{H}_*)} \bigg) + \text{const.} \\
    \nonumber & \qquad -\frac{S(1-\mu)}{2\lambda L_*}\bigg(1 - \frac{\lambda\alpha}{1+\mu}\bigg)\|\bm{w} - \bm{w}^*\|_2^2  \\
    \nonumber & \quad\qquad - \frac{S(1-\mu)\alpha}{4\lambda L_*}\bigg(2 - \frac{\lambda\alpha}{1+\mu}\bigg)(\bm{w} - \bm{w}^*)^\top \mathbf{H}_*^{-1}(\bm{w} - \bm{w}^*) \\
    \nonumber & \qquad\qquad + \frac{S(1-\mu)}{4(1+\mu)L_*}(\bm{w} - \bm{w}^*)^\top \mathbf{H}_*(\bm{w} - \bm{w}^*) \\
    \nonumber & \quad = - \frac{d}{2} \ln L_* + \frac{1}{2}\sum_{i=1}^d \ln\bigg(\frac{\big(2 - \frac{\lambda}{1+\mu}(\sigma_i(\mathbf{H}_*) + \alpha)\big)\big(\sigma_i(\mathbf{H}_*) + \alpha\big)}{\sigma_i(\mathbf{H}_*)} \bigg)+ \text{const.} \\
    \nonumber & \qquad - \frac{S(1-\mu)}{2\lambda}\bigg(1 - \frac{\lambda\alpha}{1+\mu}\bigg) \cdot \frac{\|\bm{w} - \bm{w}^*\|_2^2}{L_*} \\
    \nonumber & \quad\qquad - \frac{S(1-\mu)\alpha}{4\lambda}\cdot \bigg(2 - \frac{\lambda\alpha}{1+\mu}\bigg) \cdot \frac{\nabla L(\bm{w})^\top \mathbf{H}_*^{-3}\nabla L(\bm{w})}{L_*} \\
    & \qquad\qquad + \frac{S(1-\mu)}{2(1+\mu)} \cdot \frac{L(\bm{w}^*)}{L_*} + o(\|\bm{w} - \bm{w}^*\|_2^2).
\end{align}
Omitting the constant and negligible terms, we thus complete the proof of Theorem \ref{thm:sgd posterior new}. Note that we keep the $O(\alpha^2)$ term in Equation \ref{eq:log likelihood explicit form} and Theorem \ref{thm:sgd posterior new} for the sake of completeness but expect it to be negligible compared with other terms, due to the fact that $\alpha$ is typically set as a very small constant within $[0,1)$.
\end{proof}

\section{Optimal Membership-Inference Score}
\label{app:optimal membership inference sgd}

\subsection{Proof for Theorem \ref{thm:optimal mia score}}
\label{app:optimal mia score}

\begin{proof}
To derive the scoring function for an optimal membership inference, we need to compute the ratio between $p(\bm{w}|\bm{w}^*_1)$ and $p(\bm{w}|\bm{w}^*_0)$, where $\bm{w}^*_0$ (resp. $\bm{w}^*_1$) denotes a local minimum (close to $\bm{w}$) of the training loss function with respect to $\{\bm{z}_2, \ldots, \bm{z}_n\}$ (resp. $\{\bm{z}_1, \ldots, \bm{z}_n\}$). Note that we've obtained the posterior distribution of $\bm{w}$ in Theorem \ref{thm:sgd posterior new}. Therefore, the remaining task is to analyze the following terms:
\begin{align}
\label{eq:log diff posterior distribution}
    & \nonumber \ln p(\bm{w}|\bm{w}^*_1) - \ln p(\bm{w}|\bm{w}^*_0) \\
    & = \nonumber -\frac{d}{2}\big[\ln L_1(\bm{w}^*_1) - \ln L_0(\bm{w}^*_0)\big]
    + \frac{1}{2}\sum_{i=1}^d \ln\bigg(\frac{\big(2 - \frac{\lambda}{1+\mu}(\sigma_i(\mathbf{H}_1(\bm{w}^*_1)) + \alpha)\big)\big(\sigma_i(\mathbf{H}_1(\bm{w}^*_1)) + \alpha\big)}{\big(2 - \frac{\lambda}{1+\mu}(\sigma_i(\mathbf{H}_0(\bm{w}^*_0)) + \alpha)\big)\big(\sigma_i(\mathbf{H}_0(\bm{w}^*_0)) + \alpha\big)} \cdot \frac{\sigma_i(\mathbf{H}_0(\bm{w}^*_0))}{\sigma_i(\mathbf{H}_1(\bm{w}^*_1))}\bigg) \\
    & \nonumber \qquad -\frac{S(1-\mu)}{2\lambda}\bigg(1 - \frac{\lambda\alpha}{1+\mu}\bigg)\bigg(\frac{\|\bm{w} - \bm{w}^*_1\|_2^2}{L_1(\bm{w}^*_1)} - \frac{\|\bm{w} - \bm{w}^*_0\|_2^2}{L_0(\bm{w}^*_0)} \bigg) + \frac{S(1-\mu)}{2(1+\mu)}\bigg(\frac{L_1(\bm{w})}{L_1(\bm{w}^*_1)} - \frac{L_0(\bm{w})}{L_0(\bm{w}^*_0)}\bigg) \\
    & \qquad -\frac{S(1-\mu)\alpha}{4\lambda}\bigg(2 - \frac{\lambda\alpha}{1+\mu}\bigg)\bigg(\frac{\nabla L_1(\bm{w})^\top\mathbf{H}_1(\bm{w}^*_1)^{-3}\nabla L_1(\bm{w})}{L_1(\bm{w}^*_1)} - \frac{\nabla L_0(\bm{w})^\top\mathbf{H}_0(\bm{w}^*_0)^{-3}\nabla L_0(\bm{w})}{L_0(\bm{w}^*_0)}\bigg),
\end{align}
where $\mathbf{H}_0(\bm{w}_0^*)$ (resp. $\mathbf{H}_1(\bm{w}_1^*)$) denotes the Hessian of $L_0$ (resp. $L_1$) at $\bm{w}^*_0$ (resp. $\bm{w}^*_1$).
Since both $\bm{w}_0^*$ and $\bm{w}_1^*$ are close to parameters of the observed victim model $\bm{w}$, so we can approximate the corresponding loss using second-order Taylor expansion. Also, according to \Cref{asmp:hessian}, we know $\mathbf{H}_0(\bm{w}_0^*) = \mathbf{H}_1(\bm{w}_1^*) = \mathbf{H}_*$ and $L_0(\bm{w}_0^*) = L_1(\bm{w}^*_1) = L_*$.
Thus, we can simplify Equation \ref{eq:log diff posterior distribution} and obtain the following form:
\begin{align}
\label{eq:log diff posterior distribution simplified}
    & \nonumber \ln p(\bm{w}|\bm{w}^*_1) - \ln p(\bm{w}|\bm{w}^*_0) \\
    & = -\frac{S(1-\mu)}{2\lambda L_*}\bigg(1 - \frac{\lambda\alpha}{1 + \mu}\bigg) \big(\|\bm{w} - \bm{w}^*_1\|_2^2 - \|\bm{w} - \bm{w}^*_0\|_2^2\big) + \frac{S(1-\mu)}{2(1+\mu)L_*} \big(L_1(\bm{w}) - L_0(\bm{w})\big) \nonumber \\
    & \qquad - \frac{S(1-\mu)\alpha}{4\lambda L_*}\bigg(2 - \frac{\lambda\alpha}{1+\mu}\bigg)\big(\nabla L_1(\bm{w})^\top\mathbf{H}_*^{-3}\nabla L_1(\bm{w}) - \nabla L_0(\bm{w})^\top\mathbf{H}_*^{-3}\nabla L_0(\bm{w})\big) \nonumber \\
    & = -\frac{S(1-\mu)}{2\lambda L_*}\bigg(1 - \frac{\lambda\alpha}{1 + \mu}\bigg) \big(\nabla L_1(\bm{w})^\top \mathbf{H}_*^{-1} \mathbf{H}_*^{-1} \nabla L_1(\bm{w}) - \nabla L_0(\bm{w})^\top \mathbf{H}_*^{-1} \mathbf{H}_*^{-1} \nabla L_0(\bm{w}) \big) +  \frac{S(1-\mu)\ell(\bm{w},\bm{z}_1)}{2n(1+\mu)L_*} \nonumber \\
    & \qquad - \frac{S(1-\mu)\alpha}{4n\lambda L_* }\bigg(2 - \frac{\lambda\alpha}{1+\mu}\bigg)\bigg( 2\nabla L_0(\bm{w})^\top\mathbf{H}_*^{-3}\nabla\ell(\bm{w}, \bm{z}_1) + \frac{1}{n} \ell(\bm{w}, \bm{z}_1)^\top\mathbf{H}^{-3}_*\nabla\ell(\bm{w}, \bm{z}_1)  \bigg) \nonumber \\
    & = -\frac{S(1-\mu)}{2n\lambda L_*}\bigg(1 - \frac{\lambda\alpha}{1 + \mu}\bigg)\bigg( 2\nabla L_0(\bm{w})^\top\mathbf{H}_*^{-1}\mathbf{H}_*^{-1}\nabla\ell(\bm{w}, \bm{z}_1) + \frac{1}{n} \|\mathbf{H}^{-1}_*\nabla\ell(\bm{w}, \bm{z}_1)\|_2^2  \bigg) +  \frac{S(1-\mu) \ell(\bm{w},\bm{z}_1)}{2n(1+\mu)L_*} \nonumber \\
    & \qquad - \frac{S(1-\mu)\alpha}{4n\lambda L_*}\bigg(2 - \frac{\lambda\alpha}{1+\mu}\bigg)\bigg( 2\nabla L_0(\bm{w})^\top\mathbf{H}_*^{-3}\nabla\ell(\bm{w}, \bm{z}_1) + \frac{1}{n} \ell(\bm{w}, \bm{z}_1)^\top\mathbf{H}^{-3}_*\nabla\ell(\bm{w}, \bm{z}_1)  \bigg),
\end{align}
where the second equality holds because of the Taylor approximation $\nabla L_i(\bm{w}) - \nabla L_i(\bm{w}^*_i) = \mathbf{H}_* (\bm{w} - \bm{w}^*_i)$ for $i\in\{0, 1\}$. Moreover, according to Lemma \ref{lem:optimial MI likelihood}, we know the optimal membership inference is given by:
\begin{align}
\label{eq:optimal membership inference general form}
    \mathcal{M}(\bm{w}, \bm{z}_1) = \mathbb{E}_{\mathcal{T}} \bigg[\sigma\bigg(\ln\bigg(\frac{p(\bm{w}| m_1=1, \bm{z}_1, \mathcal{T})}{p(\bm{w}| m_1=0, \bm{z}_1, \mathcal{T})}\bigg) + \ln\bigg(\frac{\gamma}{1-\gamma}\bigg)\bigg)\bigg],
\end{align}
where $\sigma(u) = (1+\exp(-u))^{-1}$ is the Sigmoid function, $\mathcal{T} = \{\bm{z}_2,\ldots,\bm{z}_n, m_2, \ldots, m_n\}$, and $\gamma = \mathbb{P}(m_i=1)$. Plugging Equation \ref{eq:log diff posterior distribution simplified} into Equation \ref{eq:optimal membership inference general form}, we obtain
\begin{align*}
    \mathcal{M}(\bm{w}, \bm{z}_1) = \mathbb{E}_{\mathcal{T}} \bigg[\sigma \bigg(\frac{S(1-\mu)}{2nL_*}\bigg( \frac{\ell(\bm{w}, \bm{z}_1)}{(1+\mu)} - \frac{1}{\lambda}\big(I_1 + I_2 + I_3 + I_4\big) \bigg) + t_\gamma \bigg) \bigg],
\end{align*}
where $I_1, I_2, I_3, I_4$ and $t_\gamma$ are defined as:
\begin{align*}
I_1 & := \frac{1}{n}\bigg(1 - \frac{\lambda\alpha}{1+\mu}\bigg) \cdot \|\mathbf{H}_*^{-1}\nabla\ell(\bm{w}, \bm{z}_1)\|_2^2, \\
I_2 & := 2\bigg(1 - \frac{\lambda\alpha}{1+\mu}\bigg) \cdot \big(\mathbf{H}_*^{-1}\nabla L_0(\bm{w})\big)^\top \big(\mathbf{H}_*^{-1}\nabla\ell(\bm{w}, \bm{z}_1)\big), \\
I_3 & := \frac{\alpha}{2n}  \bigg(2 - \frac{\lambda\alpha}{1+\mu}\bigg)\cdot \big(\mathbf{H}_*^{-1}\nabla\ell(\bm{w}, \bm{z}_1)\big)^\top \big(\mathbf{H}_*^{-1}\big(\mathbf{H}_*^{-1}\nabla\ell(\bm{w}, \bm{z}_1)\big)\big), \\
I_4 & := \alpha \bigg(2 - \frac{\lambda\alpha}{1+\mu}\bigg) \cdot \big(\mathbf{H}_*^{-1}\nabla L_0(\bm{w})\big)^\top \big(\mathbf{H}_*^{-1}\big(\mathbf{H}_*^{-1}\nabla\ell(\bm{w}, \bm{z}_1)\big)\big), \\
t_\gamma &:= \ln\bigg(\frac{\gamma}{1-\gamma}\bigg).
\end{align*}
Therefore, we complete the proof of Theorem \ref{thm:optimal mia score}.
\end{proof}

\subsection{Connection with LOSS attack}
\label{app:connection_with_loss}
Note that while there are additional terms in our optimal membership-inference score, there is another critical difference: the loss function has its sign flipped when compared to existing results \citep{yeom_privacy_2018, sablayrolles_white-box_2019}. While this may seem counter-intuitive at first glance, we show below the addition $-(I_1+I_2+I_3+I_4)$ terms in Equation \ref{eq:optimal mia score} are expected to be negatively correlated to the loss function, leading to the proposed scoring function, in fact, aligns with the intuition of existing results.
For simplicity, we consider the setting without regularization (\ie $\alpha = 0$) and the Hessian matrix has full rank.

According to the assumption of quadratic loss around $\bm{w}^*$, we have the following observations:
\begin{align*}
    L(\bm{w}) - L_* = \frac{1}{2}(\bm{w}-\bm{w}^*)^{\top} \mathbf{H}_* (\bm{w}-\bm{w}^*) = \frac{1}{2}(\bm{w}-\bm{w}^*)^{\top} \mathbf{U}^\top \mathrm{diag}\{\sigma_1, \ldots, \sigma_d\} \mathbf{U} (\bm{w} - \bm{w}^*),
\end{align*}
where $\mathbf{U}^\top \mathrm{diag}\{\sigma_1, \ldots, \sigma_d\} \mathbf{U}$ is the eigenvalue decomposition of $\mathbf{H}$. Let $\bm{v} = \mathbf{U} (\bm{w} - \bm{w}^*)$. Since $\mathbf{U}$ is an orthonormal matrix, we know $\|\bm{v}\|_2 = \|\bm{w} - \bm{w}^*\|_2$. Thus, we obtain
\begin{align*}
    \sigma_d \cdot \|\bm{w} - \bm{w}^*\|_2^2 \leq \sigma_j \cdot \|\bm{v}\|_2^2 = 2(L(\bm{w}) - L_*) = \sum_{j=1}^d \sigma_j \cdot v_j^2 \leq \sigma_1 \cdot \|\bm{v}\|_2^2 = \sigma_1 \cdot \|\bm{w} - \bm{w}^*\|_2^2,
\end{align*}
which further suggests (provided the Hessian has full rank)
\begin{align}
\label{eq:bound l2 norm model parameters}
    \frac{1}{\sigma_1} \leq \frac{\|\bm{w} - \bm{w}_i^*\|_2^2}{2(L_i(\bm{w}) - L_*)} \leq \frac{1}{\sigma_d} \quad \text{ for any } i\in\{0,1\}.
\end{align}
Based on Assumption \ref{asmp:hessian}, we assume that the Hessian structure and the loss function value remain unchanged with and without a single record $\bm{z}_1$. We hypothesize that the ratio $\frac{1}{k} = \frac{\|\bm{w} - \bm{w}_i^*\|_2^2}{2(L_i(\bm{w}) - L_*)}$ also remains similar for $i=0$ and $i=1$, where $k\in[\sigma_d, \sigma_1]$. Therefore, we have
\begin{align}
\label{eq:bound l2 norm difference}
\frac{2\ell(\bm{w},\bm{z}_1)}{n\sigma_1} \leq \|\bm{w} - \bm{w}_1^*\|_2^2 - \|\bm{w} - \bm{w}_0^*\|_2^2 \leq \frac{2\ell(\bm{w},\bm{z}_1)}{n\sigma_d}.
\end{align}
Note that the derivation from Equation \ref{eq:bound l2 norm model parameters} to Equation \ref{eq:bound l2 norm difference} is not mathematically rigorous, but as long as the record $\bm{z}_1$ is not too deviated from the data distribution, we expect the above inequality holds.
Plugging Equation \ref{eq:bound l2 norm difference} into the log-likelihood term inside $\mathcal{M}(\bm{w}, \bm{z}_1)$ (first equality in Equation \ref{eq:log diff posterior distribution simplified} with $\alpha = 0$), we get
\begin{align}
    \label{eq:upper_bound_loss_connection}
    \nonumber \ln p(\bm{w}|\bm{w}^*_1) - \ln p(\bm{w}|\bm{w}^*_0) &= \frac{S(1-\mu)}{2L_*} \bigg(-\frac{1}{\lambda} (\|\bm{w} - \bm{w}^*_1\|_2^2 - \|\bm{w} - \bm{w}^*_0\|_2^2) + \frac{\ell(\bm{w},\bm{z}_1)}{(1+\mu)n}\bigg) \\
    &= \frac{S(1-\mu)}{2L_*} \bigg(\frac{1}{1+\mu} - \frac{2}{\lambda k}\bigg) \frac{\ell(\bm{w}, \bm{z}_1)}{n},
\end{align}
where $k$ is some real number that falls into $[\sigma_d, \sigma_1]$. 
In addition, for the case of full-rank Hessian, it is easy to see that the $i$-th eigenvalue of $\bm\Sigma$ (\Cref{eq:explicit model fluctuation sgd}) can be written as:
$$
    \underbrace{\frac{S(1-\mu)}{2\lambda L_*}}_{\text{positive}}\bigg(2 - \frac{\lambda\sigma_i}{1 + \mu}\bigg)^{-1}.
$$
Since the covariance matrix $\bm\Sigma$ is positive semi-definite and invertible, it must follow that all of its eigenvalues are positive:
$$
2 - \frac{\lambda\sigma_i}{1 + \mu} > 0 \quad\Rightarrow\quad \frac{1}{1 + \mu} - \frac{2}{\lambda\sigma_i} < 0, \quad \text{ for any } i=1,2,\ldots,d.
$$
With the above inequalities in mind, by looking at \Cref{eq:upper_bound_loss_connection}, we get:
\begin{align}
    \ln p(\bm{w}|\bm{w}^*_1) - \ln p(\bm{w}|\bm{w}^*_0) = \underbrace{\frac{S(1-\mu)}{2nL_*}}_{>0} \underbrace{\bigg(\frac{1}{1+\mu} - \frac{2}{\lambda k}  \bigg)}_{<0} \frac{\ell(\bm{w}, \bm{z}_1)}{n}.
\end{align}
The upper limit on the score (hence the score itself) corresponding to \miattackabbr is thus proportional to the negative of the loss function, aligning with intuition (lower loss indicative of overfitting, and thus the record being a member). The score inside \miattackabbr can thus be interpreted (up to some approximation error) as $-f(\bm{w}, \bm{z}_1)\ell(\bm{w}, \bm{z}_1)$ for some $f(\bm{w}, \bm{z}_1) >0$ that essentially accounts for SGD training dynamics, and is a function dependent on parameter access to the target model.

\section{Purchase-100(S) v/s Purchase-100}
\label{app:purchase100s_vs_purchase100}
Model trainers, under practical settings, would not want to produce sub-optimal models. Under the given experimental settings (access to Purchase100 dataset), it is thus crucial to simulate model training setups that would maximize performance. The switch from Purchase-100(S) to Purchase-100 not only improves model performance but also reduces the performance of MIA attacks (\Cref{tab:purchase100s_results}).

\begin{table}[!ht]
\caption{Performance of various attacks on Purchase-100(S) and Purchase-100. There is a clear drop in performance when shifting from Purchase-100(S) to Purchase-100. Statistics for \miattackabbr (CG) are computed on 10000 samples, not the entire dataset.}
\centering
\begin{tabular}{lcrrcrrcrr}
    \toprule
    \multirow{4}{*}{Attack} & \multicolumn{3}{c}{\multirow{2}{*}{Purchase-100}}
    & \multicolumn{6}{c}{Purchase-100(S)} \\
    \cmidrule(lr){5-10} 
    & & & & \multicolumn{3}{c}{MLP-2} &  \multicolumn{3}{c}{MLP-4}
    \\
    \cmidrule(lr){2-4} \cmidrule(lr){5-7} \cmidrule(lr){8-10} 
    & \multirow{3}{*}{AUC} & \multicolumn{2}{c}{\% TPR@FPR} & \multirow{3}{*}{AUC} & \multicolumn{2}{c}{\% TPR@FPR} & \multirow{3}{*}{AUC} & \multicolumn{2}{c}{\% TPR@FPR}\\
    \cmidrule(lr){3-4} \cmidrule(lr){6-7} \cmidrule(lr){9-10}
    & & 1\% & 0.1\% & & 1\% & 0.1\% & & 1\% & 0.1\%\\
    \midrule
    LOSS & .529 $_{\pm.001}$ & 0.97 & 0.00 & .589 $_{\pm .003}$ & 1.04 & 0.00 & .666 $_{\pm .005}$ & 0.00 & 0.00\\
    LiRA & .634 $_{\pm .003}$ & 4.70 & 0.98 & .743 $_{\pm .006}$ & 9.56 & 2.79 & .843 $_{\pm .004}$ & 25.17 &  9.09 \\
    \midrule
    \miattackabbr (Ours) & .703 $_{\pm .004}$ & 13.69 & 7.52 & .791 $_{\pm .003}$ & 19.95 & 18.99 & --- & --- & ---\\
    \miattackabbr (CG) & .701 $_{\pm .009}$ & 13.72 & 7.55 & .791 $_{\pm .005}$ & 20.09 & 19.09 & .691 $_{\pm .007}$ & 1.14 & 0.16\\
    \bottomrule
\end{tabular}
\label{tab:purchase100s_results}
\end{table}

For instance, AUC values for LOSS drop from $\sim0.59$ to $\sim0.53$, and LiRA-Online from $\sim0.74$ to $\sim0.63$. While the smaller version of the dataset has recently been argued not to be very relevant \cite{carlini_membership_2022}, we believe this larger version is still interesting to study since such large datasets are practically relevant. We hope that researchers will aim to use the larger version of the dataset and, in general, train target models to maximize performance (as any model trainer would) within the constraints of their experimental design.

\subsection{\miattackabbr performance on MLP-4}
\label{app:picking epsilon}

For the MLP-4 architecture on the Purchase-100(S) dataset, we observe a significant drop in performance when using \miattackabbr compared to LiRA. We hypothesize that this performance degradation may be due to a mismatch between the damping factor $\epsilon$ used in our experiments ($2e^{-1}$) and the optimal damping factor for the eigenvalue distribution of this larger model. To test this hypothesis, we increased $\epsilon$ to $5e^{-1}$ and repeated the evaluation with \miattackabbr.

This adjustment led to a substantial improvement in performance: the AUC increased to 0.768, with TPRs of 13.11\% and 12.12\% at 1\% and 0.1\% FPR, respectively. This result supports our hypothesis and highlights a current limitation of \miattackabbr. Specifically, the damping factor $\epsilon$ should be scaled according to the model architecture—ideally determined by some percentile of the eigenvalues. However, identifying the optimal scaling method before conducting the audit remains an open question for future work.

It is important to note that further increasing $\epsilon$ does not result in a linear performance improvement; in fact, performance declines across architectures when $\epsilon$ becomes too large, which is expected. The remaining performance gap is likely due to the absence of reference models or a violation of our assumptions about Hessian behavior. However, the superior TPR at very low (0.1\%) FPRs, compared to LiRA, suggests that the former is more likely the cause.

\section{Implementing the \miattackfull}
\label{app:our_attack_imp_details}

For some given record $\bm{z}_1$, $\nabla L_0(\bm{w})$ can be computed by considering all data (except the target record) for which membership is known. To make this step computationally efficient for an audit, we pre-compute $\nabla L_1(\bm{w})$. Then, if the test record is indeed a member, we can compute $\nabla L_0(\bm{w})$ as $\nabla L_1(\bm{w}) - \frac{\nabla \ell(\bm{w}, \bm{z}_1)}{n}$. Note that this is equivalent to computing $\nabla L_0(\bm{w})$ separately for each target record.
The Hessian $\mathbf{H}_*$ is also similarly pre-computed using the model's training data.

\shortsection{\texorpdfstring{Conditioning $\mathbf{H}_*$}{Conditioning H*}}
While computing Hessian matrices for our experiments, we notice the presence of near-zero and small, negative eigenvalues (most of which are likely to arise from precision errors). Such eigenvalues make the Hessian ill-conditioned and thus cannot be inverted directly. We explore two different techniques to mitigate this: damping by adding a small constant $\epsilon$ to all the eigenvalues or a low-rank approximation where only eigenvalues (and corresponding eigenvectors) above a certain threshold $\epsilon$ are used as a low-rank approximation. We ablate over these two techniques for some candidate values of $\epsilon$. Our results (\Cref{tab:condition_h_ablation}) suggest that damping with $\epsilon=2e^{-1}$ works best across all the datasets we test, which is the setting for which we report our main results.

\begin{table}[h]
    \caption{Attack AUCs for various techniques to mitigate ill-conditioned Hessian matrix, with corresponding $\epsilon$ values.}
    \centering
    \begin{tabular}{l cccccc}
    \toprule
    \multirow{2}{*}{Dataset} & \multicolumn{3}{c}{Low-Rank} & \multicolumn{3}{c}{Damping} \\
    & $\epsilon=1e^{-2}$ & $\epsilon=1e^{-1}$ & $\epsilon=2e^{-1}$ & $\epsilon=1e^{-2}$ & $\epsilon=1e^{-1}$ & $\epsilon=2e^{-1}$ \\
    \midrule
    MNIST-Odd & .521 & .530 & .500 & .513 & .535 & .542\\
    FashionMNIST & .551 & .557 & .541 & .533 & .582 & .594\\
    \bottomrule
    \end{tabular}
    \label{tab:condition_h_ablation}
\end{table}

\section{Approximating \texorpdfstring{$L_0$}{L0}}
\label{app:approximate l0}

For auditing purposes, experiments where all but one member is known are useful, but an adversary is unlikely to have this much knowledge of the training data. We experiment with the potential use of \miattackabbr where only partial knowledge of the remaining $n-1$ members may be available to approximate $\nabla L_0$,
Approximating $L_0$ with a fraction of the actual dataset could be useful in not only reducing the computational cost of the audit, but also potentially enabling adversarial use of the attack in threat models where the attacker has partial knowledge of the training data. We evaluate \miattackabbr for versions where $L_0$ is approximated using a randomly-sampled fraction of the training data and report results in \Cref{tab:iha partial l0 p100s}.

\begin{table}[h]
\centering
\caption{Performance of approximation-based variant of \miattackabbr on Purchase100-S (MLP-2), when a fraction of data from $D \setminus \{z_1\}$ is used to approximate $L_0$. Statistics are computed on 10000 samples.}
\begin{tabular}{lcrr}
    \toprule 
    \multirow{2}{*}{Fraction} & \multirow{2}{*}{AUC} & \multicolumn{2}{c}{\%TPR@FPR} \\
    \cmidrule(lr){3-4}
    & & 1\% & 0.1\%\\
    \midrule
    0.2 & .577 $_{\pm .005}$ & 0.85 & 0.11 \\
    0.4 & .607 $_{\pm .007}$ & 1.79 & 0.28 \\
    0.6 & .638 $_{\pm .010}$ & 3.38 & 0.86 \\
    0.8 & .692 $_{\pm .012}$ & 6.47 & 1.59 \\
    0.9 & .733 $_{\pm .010}$ & 13.53 & 4.17 \\
    \midrule
    1.0 & .791 $_{\pm .005}$ & 20.09 & 19.09 \\
    \bottomrule
\end{tabular}
\label{tab:iha partial l0 p100s}
\end{table}

We see a clear degradation in performance when only a subset of data is used--this is especially true for lower fractions, where AUC can drop by as much as $\approx0.2$. More importantly, even when using 90\% of the training data, there is a significant gap in performance. The statistics we compute for \miattackabbr thus do completely utilize knowledge of all other training records. While this result suggests that adversarial use of \miattackabbr would require a very strong adversary (that posseses knowledge of nearly all training records), it also hints at how data poisoning attacks could have a large impact on downstream membership inference. A poisoning adversary could hypothetically craft data a way that interferes with $L_0$ (when relating to the optimal membership adversary) and thus increase/decrease inference risk for other records.

\clearpage
\section{Additional Results}
\label{app:additional results}

\begin{figure}[h!]
    \subfloat[Purchase-100]{%
        \includegraphics[width=.49\linewidth]{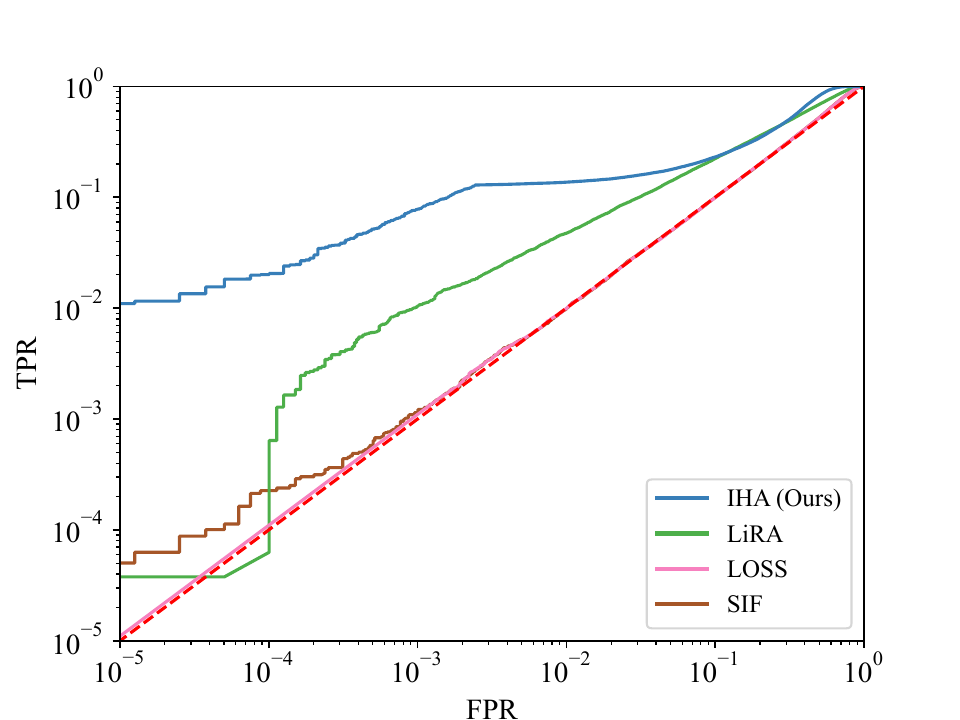}%
        \label{subfig:a}%
    }\hfill
    \subfloat[Purchase-100(S)]{%
        \includegraphics[width=.49\linewidth]{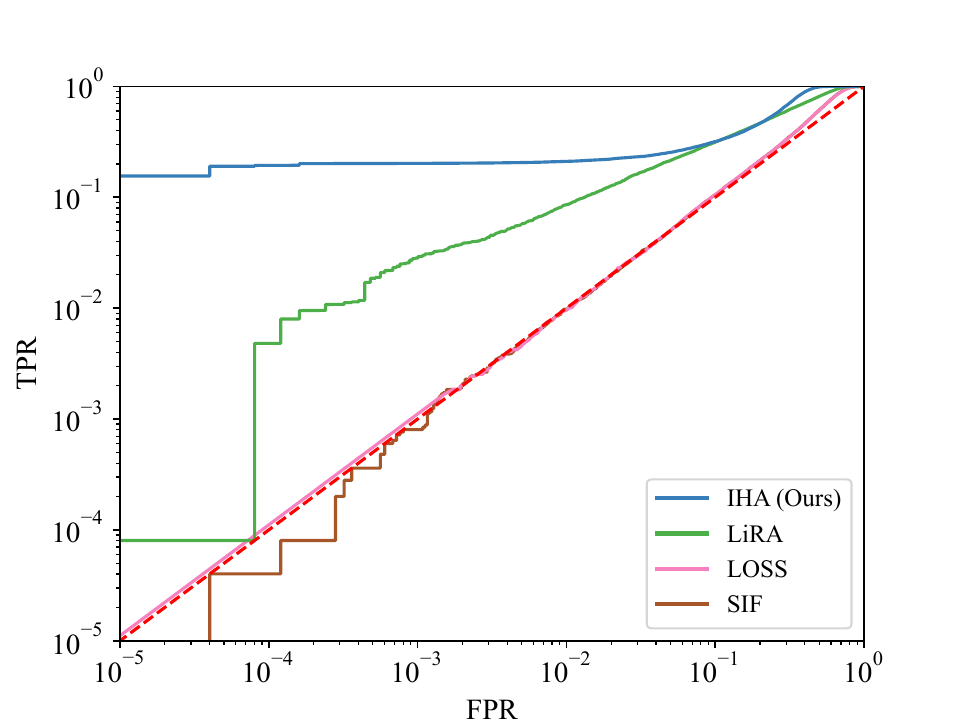}
        \label{subfig:b}%
    }
    \\
    \subfloat[MNIST-Odd]{%
        \includegraphics[width=.49\linewidth]{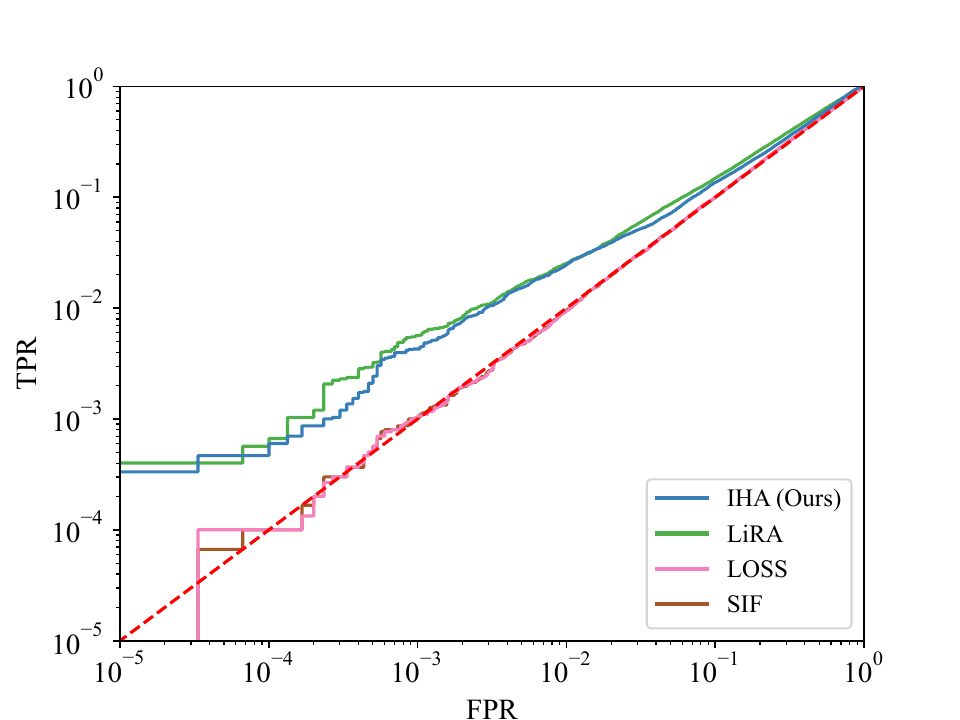}
        \label{subfig:c}%
    }\hfill
    \subfloat[FMNIST]{%
        \includegraphics[width=.49\linewidth]{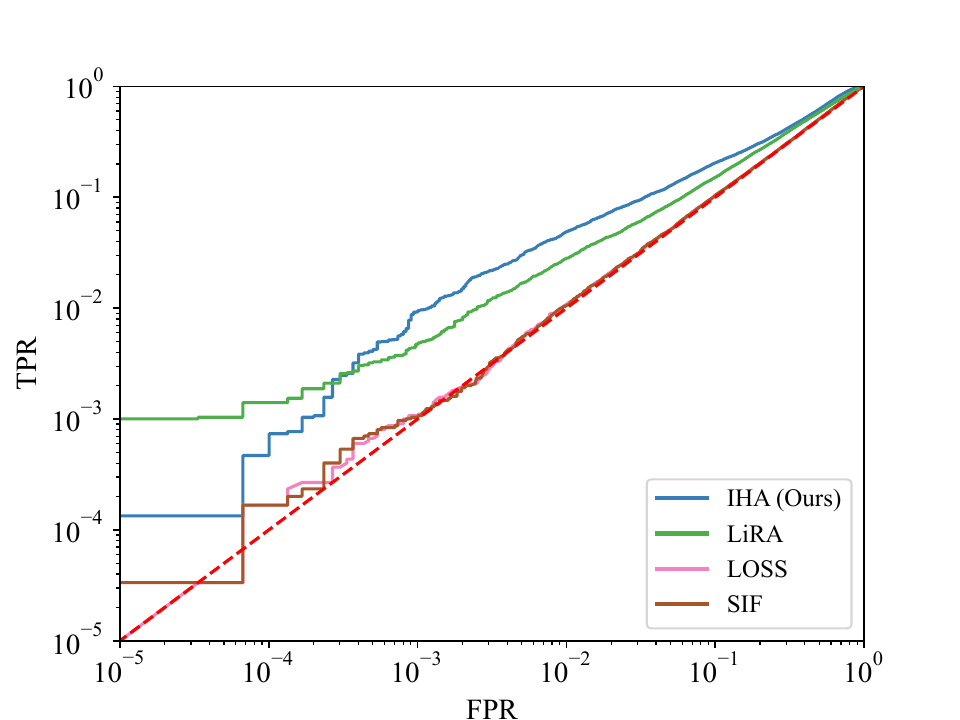}
        \label{subfig:d}%
    }
    \caption{ROC curves for low-FPR region for various attacks and datasets.}
    \label{fig:roc}
\end{figure}

\end{document}